\documentclass[sigconf]{aamas} 

\usepackage{balance} % for balancing columns on the final page

\usepackage[utf8]{inputenc} % allow utf-8 input
\usepackage[T1]{fontenc}    % use 8-bit T1 fonts
\usepackage{url}            % simple URL typesetting
\usepackage{booktabs}       % professional-quality tables
\usepackage{amsfonts}       % blackboard math symbols
\usepackage{nicefrac}       % compact symbols for 1/2, etc.
\usepackage{microtype}      % microtypography
\usepackage{xcolor}         % colors

\usepackage{amsmath}
\usepackage{mathtools}
\usepackage{amsthm}
\usepackage{dsfont}
\usepackage{enumitem}
\usepackage{upgreek}

\usepackage{caption}
\usepackage{subcaption}

\usepackage{algorithm}
\usepackage{algpseudocode}

\usepackage{lscape}
\usepackage{multirow}

\usepackage{subfiles}

\doi{CKXI1923}

\makeatletter
\gdef\@copyrightpermission{
  \begin{minipage}{0.2\columnwidth}
   \href{https://creativecommons.org/licenses/by/4.0/}{\includegraphics[width=0.90\textwidth]{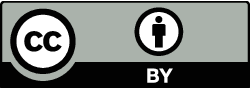}}
  \end{minipage}\hfill
  \begin{minipage}{0.8\columnwidth}
   \href{https://creativecommons.org/licenses/by/4.0/}{This work is licensed under a Creative Commons Attribution International 4.0 License.}
  \end{minipage}
  \vspace{5pt}
}
\makeatother

\setcopyright{ifaamas}
\acmConference[AAMAS '26]{Proc.\@ of the 25th International Conference
on Autonomous Agents and Multiagent Systems (AAMAS 2026)}{May 25 -- 29, 2026}
{Paphos, Cyprus}{C.~Amato, L.~Dennis, V.~Mascardi, J.~Thangarajah (eds.)}
\copyrightyear{2026}
\acmYear{2026}
\acmDOI{}
\acmPrice{}
\acmISBN{}

\title{General Flexible $f$-divergence for Challenging Offline RL Datasets with Low Stochasticity and Diverse Behavior Policies}

\author{Jianxun Wang}
\affiliation{
  \institution{North Carolina State University}
  \city{Raleigh}
  \country{United States}}
\email{jwang75@ncsu.edu}

\author{Grant C. Forbes}
\affiliation{
  \institution{North Carolina State University}
  \city{Raleigh}
  \country{United States}}
\email{gforbes@alumni.ncsu.edu}

\author{Leonardo Villalobos-Arias}
\affiliation{
  \institution{North Carolina State University}
  \city{Raleigh}
  \country{United States}}
\email{lvillal@ncsu.edu}

\author{David L. Roberts}
\affiliation{
  \institution{North Carolina State University}
  \city{Raleigh}
  \country{United States}}
\email{dlrober4@ncsu.edu}

\begin{abstract}
Offline RL algorithms aim to improve upon the behavior policy that produces the collected data while constraining the learned policy to be within the support of the dataset. However, practical offline datasets often contain examples with little diversity or limited exploration of the environment, and from multiple behavior policies with diverse expertise levels. Limited exploration can impair the offline RL algorithm's ability to estimate \textit{Q} or \textit{V} values, while constraining towards diverse behavior policies can be overly conservative. Such datasets call for a balance between the RL objective and behavior policy constraints. We first identify the connection between $f$-divergence and optimization constraint on the Bellman residual through a more general Linear Programming form for RL and the convex conjugate. Following this, we introduce the general flexible function formulation for the $f$-divergence to incorporate an adaptive constraint on algorithms' learning objectives based on the offline training dataset. Results from experiments on the MuJoCo, Fetch, and AdroitHand environments show the correctness of the proposed LP form and the potential of the flexible $f$-divergence in improving performance for learning from a challenging dataset when applied to a compatible constrained optimization algorithm.
\end{abstract}


\newcommand{\BibTeX}{\rm B\kern-.05em{\sc i\kern-.025em b}\kern-.08em\TeX}

\DeclareMathOperator*{\argmax}{argmax}

\newtheorem{theorem}{Theorem}

\newtheorem*{remark}{Remark}

\begin{document}

%%% The following commands remove the headers in your paper. For final 
%%% papers, these will be inserted during the pagination process.

\pagestyle{fancy}
\fancyhead{}

%%% The next command prints the information defined in the preamble.

\maketitle 

%%%%%%%%%%%%%%%%%%%%%%%%%%%%%%%%%%%%%%%%%%%%%%%%%%%%%%%%%%%%%%%%%%%%%%%%

% NOTE:
%   1. It's probably more accurate to describe the generalized LP form as generalized constrained optimization, because the constraints are arbitrary $f$-divergence function and most likely non-linear.

%paper paragraphs
\section{Introduction}

Recent progress in offline Reinforcement Learning (RL) significantly improved RL algorithms' practical value by alleviating the bottleneck problem of RL model training, which is the need for online environment interaction. The capability of training an RL agent on a static dataset opens up application opportunities such as AI-conversation and robot manipulation~\citep{singh2023chat, ibarz2021train}. A key idea behind the general design principle of offline RL algorithms is to optimize the learning policy with the given reward while constraining it to stay within the support of the static dataset~\citep{levine2020offline, rashidinejad2021bridging}. Algorithms designed following this pessimism principle~\citep{kumar2020conservative, kostrikov2021offline, fujimoto2021minimalist} have demonstrated considerable performance in benchmarks such as D4RL~\citep{fu2020d4rl}.

However, the implicit exploration setting and the limited combination of behavior policies suggest a gap between the characteristics of datasets from the benchmark and the real-world application. Practical considerations~\citep{murphy2001marginal, kiran2021deep} like safety and costs of experiments can restrict the practitioner who collects the data to follow an almost deterministic or even rule-based behavior. The limited exploration problem in the offline RL setting can be more severe. In addition, applications of RL may call for large-scale data collection from different practitioners or crowd-sourced data collection~\citep{walke2023bridgedata}. Such practices can lead to datasets with a high mixture of expertise levels, even for learning a single skill. 

\begin{figure*}[t]
    \centering
    \includegraphics[width=.33\textwidth]{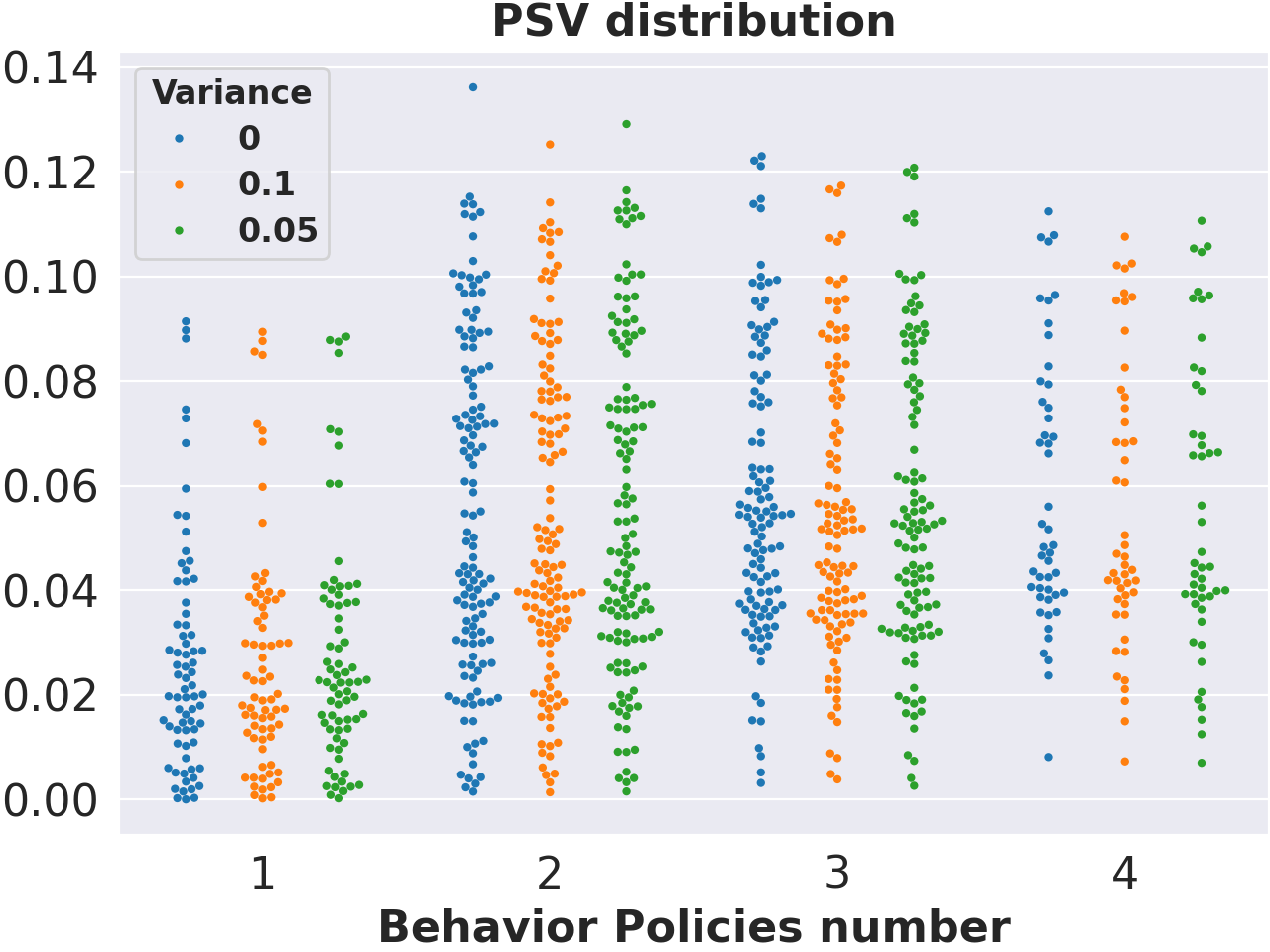}
    \includegraphics[width=.33\textwidth]{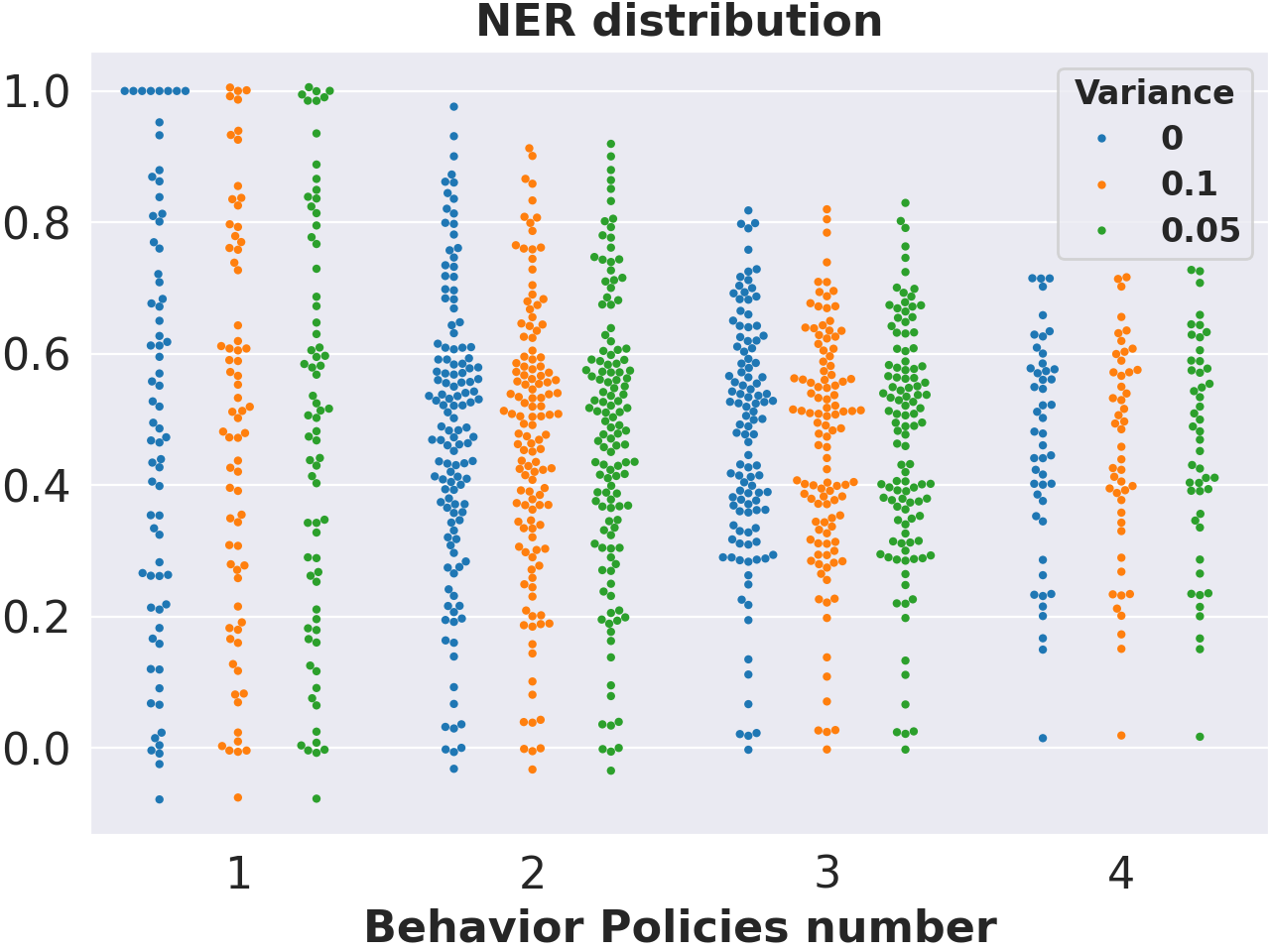}
    \includegraphics[width=.33\textwidth]{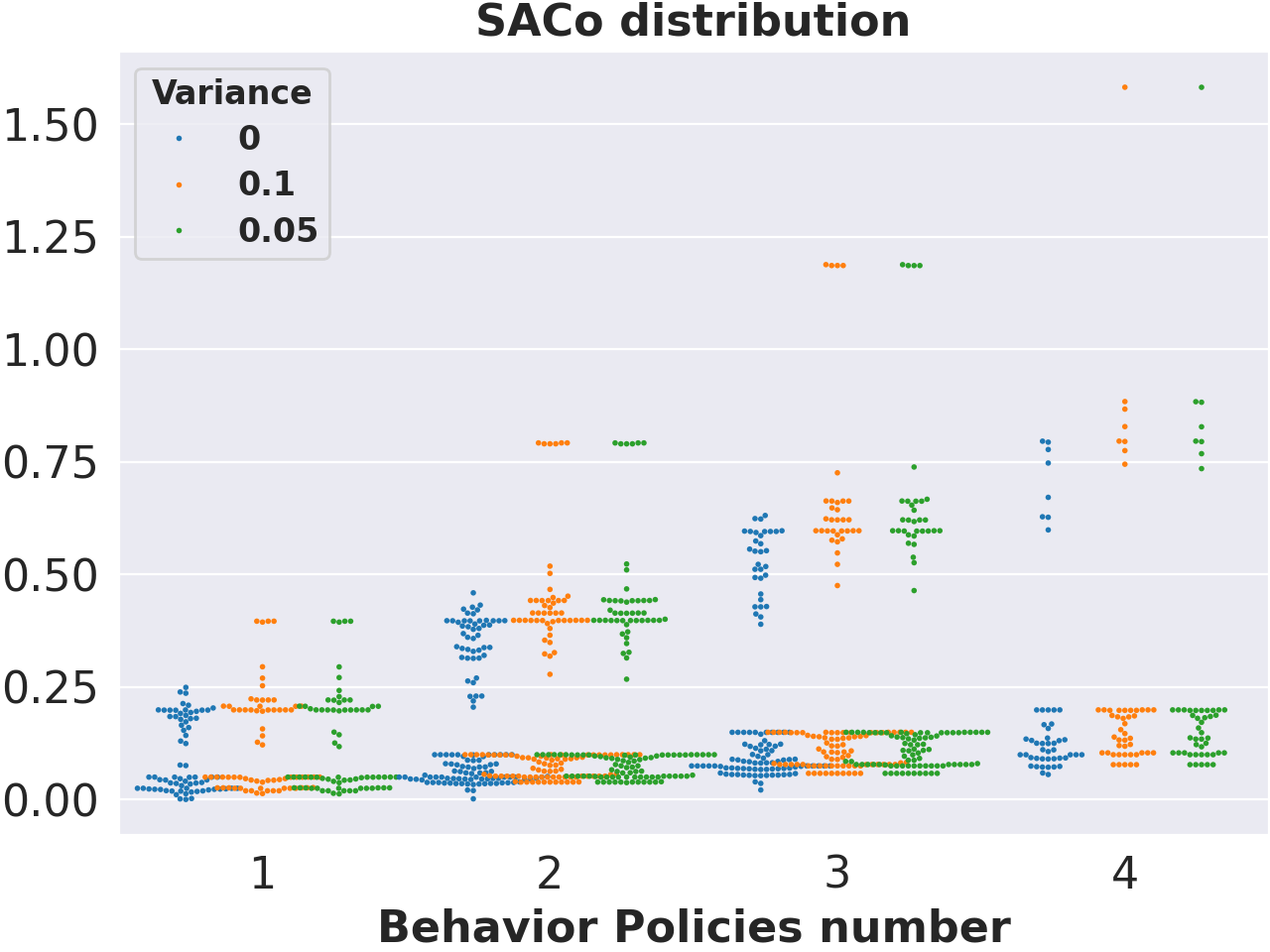}
    
    \medskip
    
    \includegraphics[width=.33\textwidth]{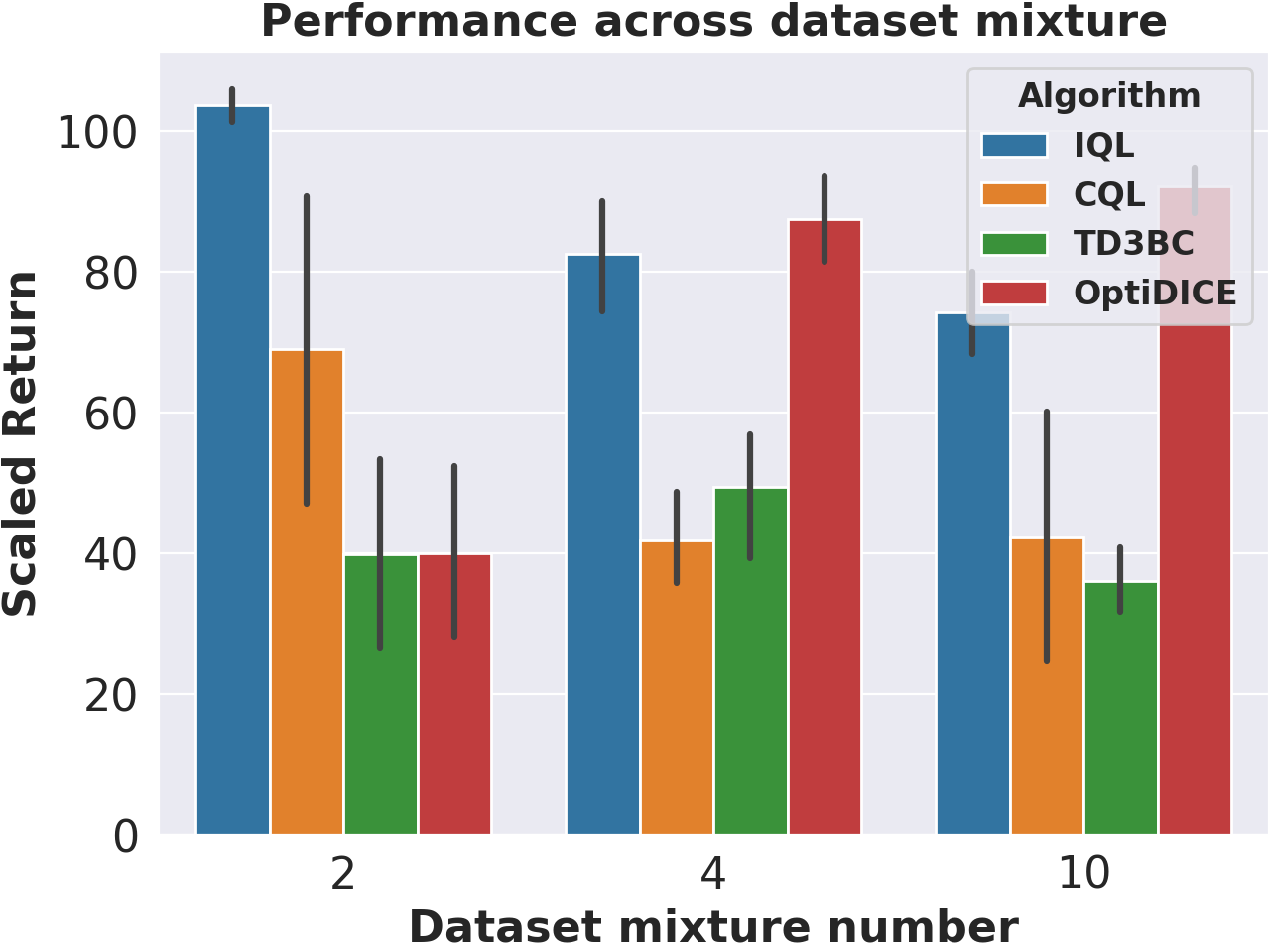}\quad
    \includegraphics[width=.33\textwidth]{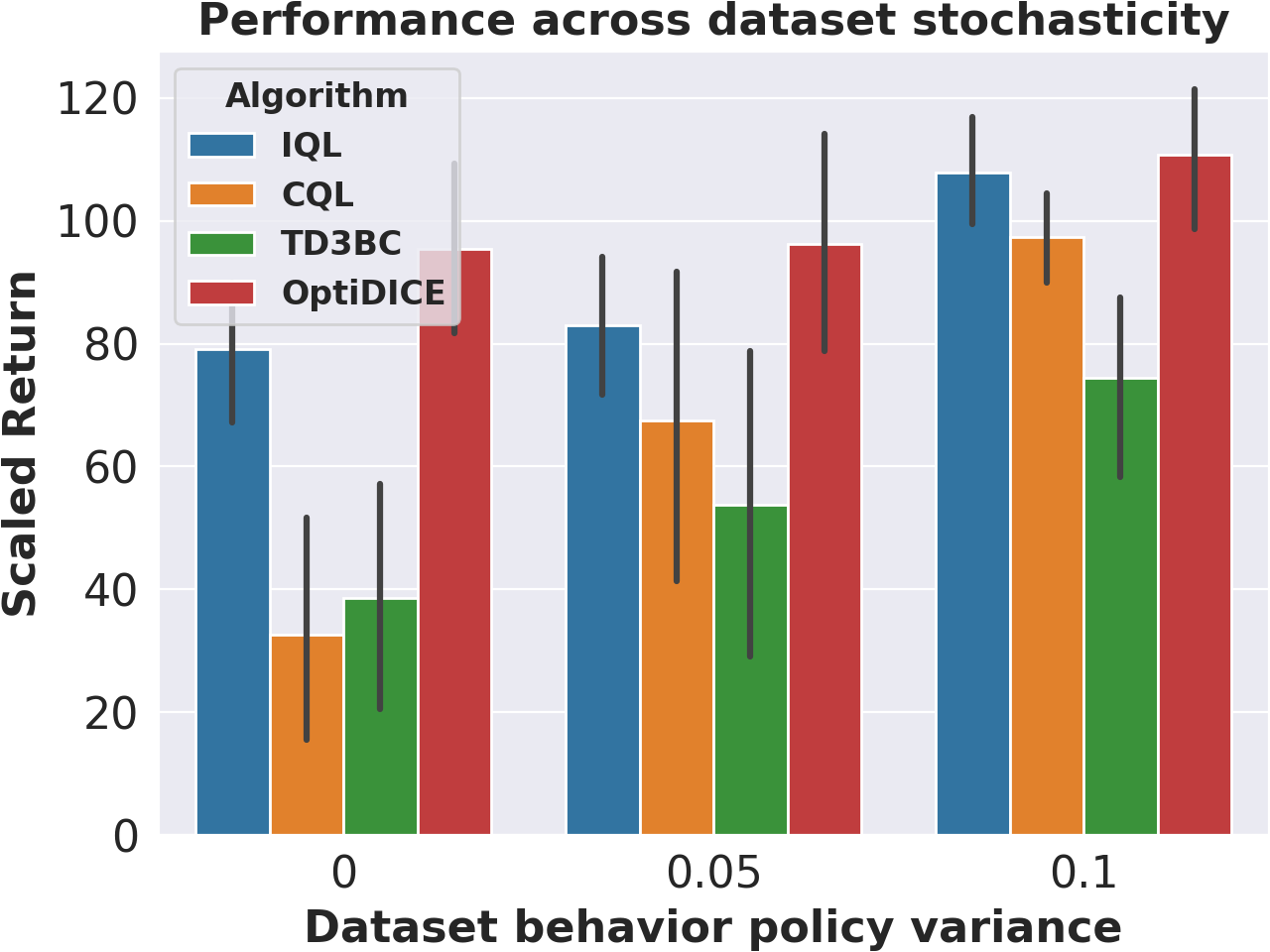}
    
    \caption{Dataset measurements across different compositions and algorithms' performance in them. (Top) Positive Scaled Variance (PSV) of the return, Normalized Expected Return (NER), and SACo as a measure of exploration; (Bottom) Algorithm performances across different dataset mixture and behavior policy stochasticity.}
    \label{fig:data_stats_and_algo_perf}
\end{figure*}
While previous works~\citep{schweighofer2021dataset, swazinna2021measuring} have explored characterizing offline RL datasets in terms of trajectories' performance and the amount of explorations present in the dataset, these measurements may only reflect dataset properties on the distribution level at best, and cannot identify any individual dataset. As the top row in Figure~\ref{fig:data_stats_and_algo_perf} shows, both measurements on dataset returns (PSV, NER) and dataset exploration (SACo) exhibit significant overlap among different settings, in terms of the number of behavior policies and behavior policy stochasticity. This makes it impossible to reliably tell what a single dataset's setting is from those measurements. Furthermore, as the bottom row in Figure~\ref{fig:data_stats_and_algo_perf} shows, each algorithm can suffer significant performance loss under different dataset mixtures and with reduced stochasticity. These observations suggest that two offline RL datasets can appear empirically similar but differ drastically in the learning algorithms' performance. 

One of the potential causes of this performance degradation can be the algorithm's universal pessimism towards the dataset. Different expertise levels and stochasticity levels of behavior policies call for different constraints. Previous work on heteroscedastic datasets~\citep{singh2022offline} and mixtures of two behavior policies~\citep{hong2023harnessing} share our opinion. The former study applied a relaxed constraint in the policy evaluation, while the latter re-weighted their learning data samples by trajectory returns. \citet{yu2023offline} focused on imitation learning by applying a relaxed $f$-divergence using a small expert dataset and a larger unlabeled dataset.

In this paper, we focus on using the $f$-divergence to control constraint level. We first highlight the implicit role of $f$-divergence for any RL algorithms that minimize the Bellman error by generalizing the unconstrained optimization derived from the Linear Programming (LP) formulation of RL and connecting it to traditional RL algorithms. We then show a unified form of LP for RL as optimizing the sum of an initial loss term and an $f$-divergence penalty on either the density ratio estimation or the estimated value, under the Bellman constraint. These theoretical results show that RL algorithms naturally incorporate the $f$-divergence against the training data, either explicitly or implicitly. Inspired by these results, we propose the flexible $f$-divergence, a general function form for $f$-divergence that is applicable to both traditional Bellman error minimization algorithms and dual optimization algorithms based on density ratio estimation, to enable flexible constraint control.  
We show that the Flexible $f$-divergence can achieve improved performance over its base algorithms in challenging settings where base algorithms fail.  

% For AAMAS 2026 full submission
% The extended version of this paper with appendix is available at https://arxiv.org/abs/xxxx.xxxxx.
% For arxiv version
This paper is the extended version of the paper submitted to AAMAS 2026, including the complete appendix.

\section{Background}

%RL problem formulation
The traditional Reinforcement Learning problem formulates the agent-environment interaction as a Markov Decision Process (MDP) \citep{sutton2018reinforcement}. The MDP can be defined as a tuple $\mathcal{M} = (\mathcal{S}, \mathcal{A}, \mathcal{T}, p_0, r, \gamma)$, with a state space $\mathcal{S}$, an action space $\mathcal{A}$, a state transition dynamics $\mathcal{T}:\mathcal{S}\times\mathcal{A}\rightarrow\Delta(\mathcal{S})$\footnote{$\Delta(\mathcal{S})$ is the space of probability distribution over $\mathcal{S}$~\cite{agarwal2019reinforcement}.}, a reward function $r: \mathcal{S}\times\mathcal{A}\rightarrow\mathbb{R}$, an initial state distribution $p_0$, and a discount factor $\gamma \in [0, 1]$. The agent's policy $\pi(a_t|s_t)$ is the conditional probability of action to take with the given state. 
With ${s_0\sim p_0, a_t\sim\pi(a_t|s_t), s_{s+1}\sim\mathcal{T}(s_{t+1}|s_t, a_t)}$, the expected discounted return is defined as $\mathbb{E}_{\pi}[\sum_{t=0}^{\infty}\gamma^{t}r(s_t, a_t)]$. Practical evaluations of the expected return for a given policy $\pi$ are the action-value function $Q^\pi(s, a)=r(s,a)+\gamma\mathbb{E}_{a'\sim\pi, s'\sim\mathcal{T}}[Q^\pi(s', a')]$ in the Bellman equation format and the state value function $V^\pi(s)=\mathbb{E}_{a\sim\pi}[Q^\pi(s, a)]$. The goal of an RL algorithm is to find the optimal policy $\pi^*$ that maximizes the expected return, $\pi^*=\argmax_\pi\mathbb{E}_{\pi}[\mathcal{R}]$, whose value function and action-value function are $V^*$ and $Q^*$.

In the offline RL setting, the RL algorithm seeks to optimize the policy using a static dataset $\mathcal{D}=\{(s, a, s', r)\}_{i=1}^N$ without access to the actual environment. To provide a measure of data sample distribution, we can express the visitation of state-action pairs as the state-action occupancy measure $d^{\pi}(s, a):=(1-\gamma)\mathbb{E}_\pi[\sum_{t}\gamma^{t}P(s_t, a_t)]$ and the state occupancy measure $d^{\pi}(s)=\sum_a d^{\pi}(s, a)$. $d^\mathcal{D}(s,a)$ and $d^\mathcal{D}(s)$ are occupancy measure of the dataset.

% RL in LP format and convert to Fenchel duality and connection to Langarange duality
The optimization of the RL problem can be formulated as a Linear Programming (LP) problem of the $V$ function ($V$-LP) with the following as the primal problem~\cite{puterman2014markov}:
\begin{align}
    \label{eq:rl_lp_primal}
    &\min_{\nu}(1-\gamma)p_0(s)\nu(s)\\
    \label{eq:rl_lp_primal_constraint}
    &\text{s.t. } (\mathcal{B}\nu)(s, a)\geq r(s, a) + \gamma (\mathcal{T}\nu)(s, a), \forall s, a
\end{align}
where $(\mathcal{T}\nu)(s, a)=\sum_{s'}\mathcal{T}(s'|s, a)\nu(s')$ and $(\mathcal{B}\nu)(s, a)=\nu(s)$.
The intuition behind the primal problem is to minimize the upper bound of $V$ since it is the fixed point solution of the Bellman equation upon convergence. Therefore the solution of $\nu(s)$ would be $V^*(s)$. The following optimization problem is the more common dual optimization form of the primal problem. 
\begin{align}
    \label{eq:rl_lp_dual}
    &\max_{d\geq 0} d(s,a)r(s, a) \\ 
    \label{eq:rl_lp_dual_constraint}
    &\text{s.t. } (\mathcal{B}_{*}d)(s)=(1-\gamma)p_0(s) + \gamma (\mathcal{T}_{*} d)(s), \forall s
\end{align}
where $(\mathcal{T}_{*} d)(s)=\sum_{\bar{s},\bar{a}}\mathcal{T}(s|\bar{s}, \bar{a})d(\bar{s}, \bar{a})$ and $(\mathcal{B}_{*}d)(s)=\sum_a d(s,a)$. The dual problem has a more natural interpretation as maximizing the expected return while ensuring the occupancy measure satisfies the Bellman flow constraint~\ref{eq:rl_lp_dual_constraint}. 

To simplify notation, we define the TD error (or equivalently, the advantage estimation) as
\begin{align}
    e_\nu(s,a) &= 
    \underbrace{r(s, a) + \gamma(\mathcal{T}\nu)(s, a)}_{Q(s, a)}
    - \nu(s).
\end{align} 
We also define the density ratio estimation of occupancy measures between $d(s, a)$ and $d^\mathcal{D}(s, a)$ as $\zeta(s, a) = \frac{d(s, a)}{d^\mathcal{D}(s, a)}$.  The LP problem results in the Lagrangian below and its dual optimization problem:
\begin{align}
    \label{eq:rl_lag_equation}
    \max_{\zeta \geq 0}\min_{\nu} & \sum_{s,a}[
        \underbrace{(1-\gamma)p_0(s)\nu(s)}_{L_P}
        + \underbrace{\zeta(s,a)r(s,a)}_{L_D} \notag\\
    + &\gamma\zeta(s,a)(\mathcal{T}\nu)(s, a) - \zeta(s,a)\nu(s)].
\end{align}
The summation $\sum_{s, a}[\cdot]$ is only possible in theory, and it implicitly assumes a universal data occupancy measure, i.e., $d^\mathcal{D}(s, a)=1$. In practical offline RL settings, we substitute it with $\mathbb{E}_{(s, a)\sim d^{\mathcal{D}}}[\cdot]$. 
To accommodate the offline RL setting, one common approach is to augment the dual problem's maximization problem (Eq.~\ref{eq:rl_lp_dual}) and correspondingly the Lagrange equation with a negative $f$-divergence between $d$ and $d^\mathcal{D}$ as $D_f(\frac{d(s, a)}{d^\mathcal{D}(s, a)})=\mathbb{E}_{d^\mathcal{D}}[f(\frac{d(s, a)}{d^\mathcal{D}(s, a)})]$. However, this added regularization creates a disconnection with the primal problem. We will show that the constraint term can be naturally derived from the primal problem under a relaxed condition and is important in the optimization when using challenging datasets.

Our theoretical results also use the conjugate of the function~\cite{boyd2004convex}
\begin{align}
    \label{eq:fenchel_conjugate}
    g(y) = \!\!\!\sup_{x\in\text{dom~} g^*} \!\!\!y*x - g^*(x).
\end{align}
If $g(x)$ is a convex function, the conjugate of its conjugate is itself. We can evaluate the optimal $x$ value by taking the RHS of Eq.~\ref{eq:fenchel_conjugate} and setting its partial derivative over $x$ to $0$.
\begin{align}
    \label{eq:optimal_conjugate_value}
    x^* = g^{*'-1}(y)
\end{align}

\section{The General LP formulation for RL}
\label{sect:general_contrained_lp_form_for_rl}

Here, we provide a general view of RL in the format of Linear Programming to highlight the role of $f$-divergence in RL, either explicitly or implicitly. We first connect the traditional RL algorithms, whose optimization aims to minimize the Bellman residual through some loss function, to the LP formulation. Then, we draw the connection from the dual problem (Eq.~\ref{eq:rl_lp_dual}) with the $f$-divergence constraint to a modified version of its primal problem (Eq.~\ref{eq:rl_lp_primal}). 

\subsection{Alternative form of $L_P$}

The original $L_P$ in RL's LP form is defined as $\alpha(s)\nu(s)$~\cite{puterman2014markov}. In Eq.~\ref{eq:rl_lp_primal}, $\alpha(s)=(1-\gamma)p_0(s)$. In a given MDP, $p_0$ is stationary and therefore $(1-\gamma)p_0(s)$ remains constant. This suggests that additional options such as $(1-\gamma)\nu(s)$ may also be valid for $L_P$.

A less obvious form for $L_P$ is $\mathbb{E}_{(s, a)\sim\mathcal{D}}[-e_\nu(s, a)]$, or simplified as $-e_\nu(s, a)$ since the offline algorithm will sample from the dataset by definition. By comparing the value function of the dataset $V^{\mathcal{D}}(s)$ and $\nu(s)$, we can derive $e_\nu(s,a)$ and transform it to a valid $L_P$ option using the performance difference lemma~\cite{agarwal2019reinforcement}:
\begin{align}
    \label{eq:rl_perform_diff_lemma}
    &V^{\mathcal{D}}(s) - \nu(s) = \frac{1}{1-\gamma}\mathbb{E}_{(s, a)\sim\mathcal{D}}[e_\nu(s, a)] \\
    \label{eq:ev_as_Lp_equation}
    &(1-\gamma)\nu(s) = (1-\gamma)V^{\mathcal{D}}(s) - \mathbb{E}_{(s, a)\sim\mathcal{D}}[e_\nu(s, a)].
\end{align}
Since the LHS of Eq.~\ref{eq:ev_as_Lp_equation} is valid for $L_P$, its RHS is equivalently valid. While $V^{\mathcal{D}}(s)$ may be unknown, it is a fixed constant for a given static dataset. This provides us with a strong analysis tool for later Sections. From a vector perspective of $\alpha(s)$, this is also suggested in~\citep{puterman2014markov} by choosing $\alpha=(\mathcal{I}-\gamma\mathcal{T})^T$. 

\subsection{Connecting Bellman minimization to LP}
\label{sect:connection_bellman_minimization_to_lp}

We start by rewriting the Lagrangian in Eq.~\ref{eq:rl_lag_equation} into a more compact form with the removal of $\zeta\geq0$ and the inclusion of $f$-divergence using a valid function $g^*(\cdot)$:
\begin{align}
    \label{eq:gen_rl_lp_lag}
    \mathcal{L}(\nu, \zeta) = \sum_{s, a}\alpha(s)\nu(s) + \zeta(s, a)e_\nu(s,a) - g^*(\zeta(s, a))
\end{align}
Previous studies~\cite{lee2021optidice, sikchi2023dual} show that the dual optimization of Eq.~\ref{eq:gen_rl_lp_lag} as $\max_\zeta\min_\nu{L}(\nu, \zeta)$ can be converted into an unconstrained optimization problem either using the convex conjugate or the solution to its Lagrange dual function as follows:
\begin{align}
    \max_\zeta\min_\nu\mathcal{L}(\nu, \zeta) = \min_\nu \alpha(s)\nu(s) + g(e_\nu(s, a)).
\end{align}
By substituting $\alpha(s)\nu(s)$ with $-e_\nu(s, a)$, we can derive the following optimization objective:
\begin{align}
    \label{eq:rl_lp_primal_opt_ineq_conjugate_2}
    \min_\nu -e_\nu(s,a) + g(e_\nu(s,a)).
\end{align}
\begin{theorem}
    \label{theorem:g_m_x_as_loss}
    For a convex $g(\cdot)$ with $g^*(1)=0$ and $g^{*'}(1)=0$, $-x+g(x)$ is convex and $\min_x g(x)-x=0$ when $x=0$
\end{theorem}
\begin{proof}
    Because $g(\cdot)$ is convex, $g'(\cdot)$ is monotonic and therefore $g'(\cdot)-1$ is monotonic. The minimum value of $-x+g(x)$ will be $0$ at $g'(0)=1$, from $g^{*'}(1)=0$, and $x^*=0$.
\end{proof}
Theorem~\ref{theorem:g_m_x_as_loss} suggests that the optimization objective in Eq.~\ref{eq:rl_lp_primal_opt_ineq_conjugate_2} is equivalent to minimizing $\nu$ such that the Bellman error, i.e. $e_\nu$, is 0, with a careful selection of $g(\cdot)$ and its corresponding conjugate function $g^*(\cdot)$. This also suggests that any existing RL algorithm, as long as it uses a compatible loss function $g(\cdot)$ to minimizing the Bellman error, has an implicit $f$-divergence implication between the value function in learning and observed data. For example, the traditional MSE loss for minimizing the Bellman error in Q-learning is equivalent to setting $g(e_\nu)=\frac{1}{2}e_\nu^2+e_\nu$, which is the conjugate of the $\chi^2$-divergence function. 

A key distinction here compared against prior works is the removal of $\zeta \geq 0$. This removal corresponds to changing the inequality constraint in the primal problem constraint (Eq.~\ref{eq:rl_lp_primal_constraint}) into the equality constraint.
Note that the constraint change will still yield a valid optimization problem because the optimal value function must satisfy the Bellman equation, which leads to the Bellman error being zero.
Such a relaxation provides theoretical justification for prior works with relaxed positivity ~\cite{kim2024relaxed} and opens up additional possible function forms for $f$-divergence.

\begin{remark}
    We follow the definition of $f$-divergence with domain over $\mathbb{R}$~\cite{wang2025glivenko, terjek2021moreau, borwein1993partially}, instead of $[0, +\infty)$. This allows negative probability measures and thus $\zeta < 0$. However, 
    $\zeta < 0$ makes the importance sampling interpretation of the dual problem invalid in the traditional probability. Recent work about the likelihood ratio estimation in quasiprobability may provide some insight about its interpretation ~\cite{drnevich2024neural}.
\end{remark}

\subsection{Unified form of LP for RL}
Here, we show a unified form of LP for RL by connecting the dual problem with an $f$-divergence augmentation to its primal problem with a similar augmentation. 
We start by rewriting the dual problem in Eq.~\ref{eq:rl_lp_dual} and~\ref{eq:rl_lp_dual_constraint} using density ratio estimation $\zeta(s, a)$
\footnote{The following conversion is straightforward in the theoretical setting, i.e. $d^{\mathcal{D}}(s,a)=1$. In the practical sampling setting, $(\mathcal{T}_{*} \zeta)(s)
=\sum_{\bar{s},\bar{a}}\mathcal{T}(s|\bar{s}, \bar{a})d^\mathcal{D}(\bar{s}, \bar{a})\zeta(\bar{s}, \bar{a})
=\mathbb{E}_{(\bar{s},\bar{a})\sim\mathcal{D}}[\mathcal{T}(s|\bar{s}, \bar{a})\zeta(\bar{s}, \bar{a})]$.
$(\mathcal{B}_{*}\zeta)(s)
=\sum_a d^{\mathcal{D}}(s,a)\zeta(s,a)
=\mathbb{E}_{a\sim\pi_\beta(\cdot|s)}[d^\mathcal{D}(s)\zeta(s,a)]$, which would transform to full sampling when aggregate over all state $s$, i.e. $\sum_s(\mathcal{B}_{*}\zeta)(s) = \mathbb{E}_{(s,a)\sim\mathcal{D}}[\zeta(s,a)]$.
}
and the standard $L_P$ term with the inclusion of $f$-divergence:
\begin{align}
    \label{eq:rl_lp_dual_reform}
    &\max_\zeta \zeta(s,a)r(s, a) -g^*(\zeta(s,a))\\ %\textcolor{red}{-D_{f}(\zeta(s, a))}
    \label{eq:rl_lp_dual_reform_constraint}
    &\text{s.t. } (\mathcal{B}_{*}\zeta)(s) - \alpha(s) - \gamma (\mathcal{T}_{*} \zeta)(s)=0, \forall s,a
\end{align}
The constraint is still the Bellman flow constraint, but rewritten in the format of a residual equal to zero. We define $e_\zeta(s,a)=(\mathcal{B}_{*}\zeta)(s) - \alpha(s) - \gamma (\mathcal{T}_{*} \zeta)(s)$ as the Bellman occupancy residual. Following Section~\ref{sect:connection_bellman_minimization_to_lp}, its Lagrangian is Eq.~\ref{eq:gen_rl_lp_lag}, and it can be transformed into the following unconstrained optimization problem with Eq.~\ref{eq:rl_lp_primal_opt_ineq_conjugate_2} being its special case:
\begin{align}
    \label{eq:unified_lp_dual_unconstrained}
    \min_\nu \alpha(s)\nu(s) + g(e_\nu(s,a)).
\end{align}
We can also rewrite it into a singular unconstrained optimization problem\footnote{\label{ft:zero_constraint_function} We define 
$\delta_{0}(x)=\begin{cases}
    0\quad\text{if }x=0,\\ \infty\quad\text{otherwise}
\end{cases}$.}:
\begin{align}
    \label{eq:unified_lp_dual_singular}
    \max_\zeta \sum_{s,a}\zeta(s,a)r(s, a) -g^*(\zeta(s,a)) - \delta_0(e_\zeta(s,a))
\end{align}
Now consider a modified primal problem with a similar augmentation of a convex function $\varphi(\nu)$ and conversion of the inequality constraint to the equality constraint:
\begin{align}
    \label{eq:rl_lp_primal_reform}
    &\min_{\nu}\alpha(s)\nu(s) + \varphi(\nu(s))\\
    \label{eq:rl_lp_primal_reform_constraint}
    &\text{s.t. } e_\nu(s,a)=0, \forall s, a 
\end{align}
The min-max optimization over its Lagrangian can be written as
\begin{align}
    \max_\zeta\min_\nu &\mathcal{L}(\nu,\zeta)=\max_\zeta\min_\nu\sum_{s,a}\alpha(s)\nu(s)+\zeta(s,a)r(s,a)
    \notag \\
    & + \zeta(s,a)\cdot\gamma(\mathcal{T}\nu)(s,a)-\zeta(s,a)\cdot(\mathcal{B}\nu)(s,a) 
    \notag\\
    &+ \varphi(\nu(s)) \\
    \label{eq:unified_lp_primal_negate_step}
    =& \min_\zeta\max_\nu\sum_{s,a}-\zeta(s,a)r(s,a) -\alpha(s)\nu(s)
    \notag\\
    &- \nu(s)\cdot\gamma(\mathcal{T}_{*}\zeta)(s)+\nu(s)\cdot(\mathcal{B}_{*}\zeta)(s)
    \notag\\
    &- \varphi(\nu(s))\\
    =&\min_\zeta\sum_{s,a}-\zeta(s,a)r(s,a)
    \notag \\
    &+\max_\nu [\nu(s)\cdot
    \underbrace{\bigl((\mathcal{B}_{*}\zeta)(s)-\alpha(s)-\gamma(\mathcal{T}_{*}\zeta)(s)\bigr)}_{e_\zeta(s,a)}
    - \varphi(\nu(s))] \\
    =&\min_\zeta\sum_{s,a}-\zeta(s,a)r(s,a) +\max_\nu [\nu(s)\cdot e_\zeta(s,a) - \varphi(\nu(s))] \\
    \label{eq:unified_lp_primal_conjugate_step}
    =& \min_\zeta\sum_{s,a}-\zeta(s,a)r(s,a)+\varphi^*(e_\zeta(s,a)) \\
    =& \max_\zeta\sum_{s,a}\zeta(s,a)r(s,a)-\varphi^*(e_\zeta(s,a))
    \label{eq:unified_lp_primal_unconstrained}
\end{align}
Eq.~\ref{eq:unified_lp_primal_negate_step} changes the min-max optimization direction by negating the Lagrangian. Eq.~\ref{eq:unified_lp_primal_conjugate_step} is the result of applying convex conjugate. 

Eq.~\ref{eq:unified_lp_primal_unconstrained} shares a close similarity with Eq.~\ref{eq:unified_lp_dual_unconstrained}. We can also rewrite the constrained optimization problem in Eq.~\ref{eq:rl_lp_primal_reform} and~\ref{eq:rl_lp_primal_reform_constraint} into a singular unconstrained optimization problem, similar to Eq.~\ref{eq:unified_lp_dual_singular}:
\begin{align}
    \label{eq:unified_lp_primal_singular}
    \min_{\nu}\sum_{s,a}\alpha(s)\nu(s) + \varphi(\nu(s))+\delta_0(e_\nu(s,a))
\end{align}
It is straightforward to spot the similarity between
Eq.~\ref{eq:unified_lp_dual_unconstrained}
and~\ref{eq:unified_lp_primal_singular}, as well as the similarity between
Eq.~\ref{eq:unified_lp_dual_singular}
and~\ref{eq:unified_lp_primal_unconstrained}.
In fact, with either
$g(\cdot)=\delta_0(\cdot)$ or $\varphi^*(\cdot)=\delta_0(\cdot)$, {\em i.e.}, applying
the equality constraint to either $e_\nu$ or $e_\zeta$, we can derive the
following equivalence:
\begin{align}
    &\min_{\nu}\sum_{s,a}\alpha(s)\nu(s) + \varphi(\nu(s))+g(e_\nu(s,a)) 
    \notag \\
    \equiv &\max_\zeta \sum_{s,a}\zeta(s,a)r(s, a) - \varphi^*(e_\zeta(s,a))-g^*(\zeta(s,a)) 
\end{align}
From the convex conjugate perspective, the equivalence holds because the conjugate function of $\delta_0(\cdot)$ is $\delta_0^*(\cdot)=0$. From the constrained optimization using Lagrangian, the equivalence holds because any valid solutions need to have $e_\nu(s,a)=0$ or $e_\zeta(s,a)=0$.
By applying the equality constraint on $e_\nu(s, a)$, the primal problem is equivalent to an unconstrained version of the dual problem, with its equality constraint on $e_\zeta(s,a)$ changed into a penalty function characterized by $\varphi^*(\cdot)$. {\em Vice versa} for the dual problem. 
As a result, the selection of the penalty function, either $g(\cdot)$ or $\varphi^*(\cdot)$, and its conjugate function directly impact the optimization process.

\section{Flexible $f$-divergence}

Theoretical results from the previous section show that the LP optimization of RL, either implicitly or explicitly, contains the $f$-divergence regularization by choosing the penalty function $g(\cdot)$ and its conjugate $g^*(\cdot)$. As a result, the selection of $g^*(\cdot)$ is crucial. 

Previous theoretical analyses of offline RL algorithms' performance guarantees are based on assumptions about the ``concentrability'' of a dataset's coverage~\citep{agarwal2019reinforcement, rashidinejad2021bridging, zhan2022offline}. Formally, a concentrability coefficient $C^\pi$ is defined as the smallest constant such that $\frac{d^\pi(s, a)}{d^\mathcal{D}(s, a)} \leq C^\pi, \forall \pi \in \mathcal{F}^\pi, s\in S, a\in A$. The concrete family of the policy $\mathcal{F}^\pi$ determines the concrete type of the concentrability coefficient. For example, setting it to be the family of $\pi^*$ yields the $\pi^*$-concentrability. The concentrability coefficient characterizes the lower bound of $d^\mathcal{D}$ against the policy family. Under such an assumption, the traditional application of $f$-divergence in the optimization can be over-conservative since it places the same constraint over $d^\mathcal{D}(s, a)$ from a policy outside of the policy family of our interest, usually those producing higher returns. This is more likely with more fundamentally different behavior polices. 

Here, we introduce the function form of the Flexible $f$-divergence that is generalized and flexible to apply different levels of penalties or constraints by selecting corresponding parameters and base functions. We express the function form in the context of the divergence function $g^*(\cdot)$, but we also provide concrete steps to compute the corresponding $g(\cdot)$ such that the Flexible $f$-divergence is applicable to the general form of RL algorithms with the Bellman error minimization.

\subsection{Functional Form for Flexible $f$-divergence}

We propose a general function format $g^*(\zeta; \alpha_-, \alpha_+, \beta)$, shorten for $g^*_{\alpha_\pm, \beta}(\zeta)$, as the function for the Flexible $f$-divergence by joining two scaled closed convex functions $g^*_-$ and $g^*_+$ at the threshold value $\beta$ to constrain $\zeta$ with different levels and functions. We denote its conjugate function as $g_{\alpha_\pm, \beta}(\cdot)$.
We use $\alpha_{-}$ and $\alpha_{+}$ to scale $g^*_-$ and $g^*_+$ respectively. For general optimization, we would want the function to be continuous at $\beta$, {\it i.e.}, $\alpha_{-}g^*_{-}(\beta) = \alpha_{+}g^*_{+}(\beta)$. To ensure $g^*(\zeta; \alpha_-, \alpha_+, \beta)$'s conjugate is also a meaningful optimization objective, both $g^{*'}_{-}$ and $g^{*'}_{+}$ must be invertible. Lastly, its Fenchel conjugate needs to be continuous at $\beta$, which requires $\alpha_{-}g^{*'}_{-}(\beta) = \alpha_{+}g^{*'}_{+}(\beta)$. 

We start by picking two closed, strictly convex and differentiable functions $\bar{g}^*_{-}$ and $\bar{g}^*_{+}$ from the family of valid functions for $f$-divergence, {\it i.e.}, $\mathcal{F}=\{f|f:\mathbb{R}^+\rightarrow \mathbb{R}, f(1)=0, f^{'}(1)=0\}$. The condition follows Theorem~\ref{theorem:g_m_x_as_loss} to ensure the corresponding $-x+g_{\alpha_\pm, \beta}(x)$ is a full fledged loss function. The choice of functions guarantees their derivative functions to be monotonic and thus invertible. Based on these two base functions, we apply scaling and a linear difference: 
\begin{equation}
    \label{eq:adapt_f_div_func}
    g^*_{\alpha_\pm, \beta}(\zeta) = \begin{cases}
         \alpha_{-}\bar{g}^*_{-}(\zeta) - \mathds{1}_{\beta < 1}(k_g\zeta+C_g) & \zeta<\beta\\
         \alpha_{+}\bar{g}^*_{+}(\zeta) + \mathds{1}_{\beta \geq 1}(k_g\zeta+C_g) & \zeta\geq\beta.
    \end{cases}
\end{equation}

Depending on whether $\beta < 1$ or not, we apply a linear difference to either $\bar{g}^*_{-}$ or $\bar{g}^*_{+}$. The function without added linear difference ($g^*_{\pm}(\zeta)=\bar{g}^*_{\pm}(\zeta)$) will retain its function form scaled by its coefficient and still satisfy the requirement of being a function for $f$-divergence, especially $g^*_{\pm}(1)=0$.\footnote{We use $\bar{g}^*_{\pm}$ to indicate either $\bar{g}^*_-$ or $\bar{g}^*_+$.} The function with added linear difference ($g^*_{\pm}(\zeta)=\bar{g}^*_{\pm}(\zeta)\pm(k_g\zeta+C_g)$) is not guaranteed to behave the same as the original function or even be a valid function by itself for $f$-divergence. Evaluating $\zeta$ over the support of $d^\mathcal{D}$, we have
\begin{align}
    &\mathbb{E}_{d^\mathcal{D}}[\bar{g}^*_{\pm}(\frac{d(s, a)}{d^\mathcal{D}(s, a)})\pm (k_g \frac{d(s, a)}{d^\mathcal{D}(s, a)}+C_g)] \\
    = &\mathbb{E}_{d^\mathcal{D}}[\bar{g}^*_{\pm}(\frac{d(s, a)}{d^\mathcal{D}(s, a)})]\pm \mathbb{E}_{d^\mathcal{D}}[k_g(\frac{d(s, a)}{d^\mathcal{D}(s, a)}-1)] \notag\\
    &\pm\mathbb{E}_{d^\mathcal{D}}[C_g+k] \\
    \label{eq:f_d_bia_derivation_step3}
    = &D_{\bar{g}^*_\pm}(\zeta)\pm D_{f(x)=k_g(x-1)}(\zeta)\pm\mathbb{E}_{d^\mathcal{D}}[C_g+k]\\
    \label{eq:f_d_bia_derivation_step4}
    = &D_{\bar{g}^*_\pm}(\zeta)\pm (C_g+k).
\end{align}
We can derive Eq.~\ref{eq:f_d_bia_derivation_step4} from Eq.~\ref{eq:f_d_bia_derivation_step3} by using $f$-divergence's property of $D_f(\frac{d}{d^\mathcal{D}})=0$ iff $f(x)=c(x-1)$. From Eq.~\ref{eq:f_d_bia_derivation_step4}, we argue that the function component with added difference can work equivalently as its original $f$-divergence plus a constant bias and does not impact its usage as a learning objective.

We pick the linear coefficient $k_g=\alpha_{-}\bar{g}^{*'}_{-}(\beta) - \alpha_{+}\bar{g}^{*'}_{+}(\beta)$ to satisfy the condition of $\alpha_{-}g^{*'}_{-}(\beta) = \alpha_{+}g^{*'}_{+}(\beta)$, as illustrated by evaluating the derivative Function~\ref{eq:adapt_f_div_derivative} with $\zeta=\beta$. This introduces additional difference between $g^*_-(\beta)$ and $g^*_+(\beta)$ terms in Eq.~\ref{eq:adapt_f_div_func}. So we choose $C_g=\alpha_{-}\bar{g}^{*}_{-}(\beta) - \alpha_{+}\bar{g}^{*}_{+}(\beta) - \beta k_g$ to remove the value difference. 

For the complete dual optimization algorithms, for example, using dual gradient descent~\citep{boyd2004convex}, $g^*(\zeta)$ in Eq.~\ref{eq:adapt_f_div_func} would be enough. However, for algorithms such as OptiDICE that use intermediate $e_\nu$ to estimate its current optimal $\hat{\zeta^*}$ or algorithms directly optimizing Eq.~\ref{eq:rl_lp_primal_opt_ineq_conjugate_2}, we derive the estimation of $\hat{\zeta^*}=g^{*'-1}_{\alpha_\pm, \beta}(e_\nu)$ by following Eq.~\ref{eq:optimal_conjugate_value} to get
\begin{align}
    \label{eq:adapt_f_div_derivative}
    &g^{*'}_{\alpha_\pm, \beta}(\zeta) = \begin{cases}
        \alpha_{-}\bar{g}^{*'}_{-}(\zeta) - \mathds{1}_{\beta < 1}(k_g) & \zeta<\beta\\
        \alpha_{+}\bar{g}^{*'}_{+}(\zeta) + \mathds{1}_{\beta \geq 1}(k_g) & \zeta\geq\beta
    \end{cases} \\
    \label{eq:adapt_f_div_prime_inverse}
    &g^{*'-1}_{\alpha_\pm, \beta}(e_\nu) = \begin{cases}
        \bar{g}^{*'-1}_{-}(\frac{e_\nu + \mathds{1}_{\beta < 1}(k_g)}{\alpha_{-}}) & e_\nu < \beta_{e_\nu} \\
        \bar{g}^{*'-1}_{+}(\frac{e_\nu - \mathds{1}_{\beta \geq 1}(k_g)}{\alpha_{+}}) & e_\nu \geq \beta_{e_\nu}
    \end{cases}.
\end{align}

Since we use $\beta$ as the threshold to separate the evaluation of $\zeta$, correspondingly, there will be a threshold for $e_\nu$ as the result of $g(e_\nu)$ and $g^*(\zeta)$ being the conjugate of each other. We compute the threshold value for $e_\nu$ as $\beta_{e_\nu}$ by evaluating $g^{*'}(\zeta=\beta)$ as $\beta_{e_\nu} = \mathds{1}_{\beta \geq 1}(\alpha_{-}\bar{g}^{*'}_{-}(\beta)) + \mathds{1}_{\beta < 1}(\alpha_{+-}\bar{g}^{*'}_{+}(\beta))$. 
For algorithms directly optimizing Eq.~\ref{eq:rl_lp_primal_opt_ineq_conjugate_2}, We can then acquire the optimization objective of $e_\nu$ by substituting the variable $\zeta$ in the convex conjugate with $\hat{\zeta^*}$ as $L(\nu)=-e_\nu+\hat{\zeta^*}e_\nu-g^*_{\alpha_\pm, \beta}(\hat{\zeta^*})$, with the possible addition of other terms depending on the concrete algorithm. 

\subsection{Heuristic estimation of $\alpha_\pm$ and $\beta$}

The native form of $g^*_{\alpha_\pm, \beta}(\zeta)$ relies on the selection of $\alpha_\pm$, $\beta$, and $g^*_\pm(\cdot)$. While the fully automatic selection algorithm for these components is not the main contribution of our paper, we provide heuristic methods to estimate values for $\alpha_\pm$ and $\beta$ using the learning agent's predicted values from the offline dataset. 

\textbf{Estimating $\alpha_\pm$.} Inspired by~\citet{han2025cosine}, we train an additional behavior cloning policy $\pi_b$ from the offline dataset and compare its predicted action probability against learned $\exp(e_\nu)$. Specifically, we predict the action probability $\pi_b(a|s)$ and $\exp(e_\nu(s,a))$ for each state-action pair in the sampled batch. We treat the list of computed probability and the list of $\exp(e_\nu)$ as two vectors, i.e. $\overrightarrow{\pi_b}$ and $\overrightarrow{\exp(e_\nu)}$, and compute $\alpha_\pm$ as follows:
\begin{align}
    \alpha_+ &= \frac{1}{\cos\langle \overrightarrow{\pi_b}, \overrightarrow{\exp(e_\nu)}\rangle} \\
    \alpha_- &= \frac{1}{1 - \cos\langle \overrightarrow{\pi_b}, \overrightarrow{\exp(e_\nu)}\rangle},
\end{align}
where $\cos\langle\cdot, \cdot\rangle$ computes the cosine similarity. We use an exponential moving average (EMA) to smooth the estimation throughout training.

\textbf{Estimating $\beta$.} Because $g(e_\nu)$ and $g^*(\zeta)$ are conjugate functions to each other, $e_\nu$ and $\zeta$ has a direct correspondence. We compute the average $e_\nu$ as $\bar{e_\nu}$ from sampled batches and smooth it with EMA. We then compute the beta value by $\beta = g^{*'-1}_{\alpha_\pm, 0}(\bar{e_\nu})$.

\subsection{Effect of base function, $\alpha_\pm$ and $\beta$}

\label{subsect:adaptive_f_divergence_in_practical_algorithms}
\begin{table*}[ht]
    \centering
    \begin{tabular}{c||c|c|c|c}
         $f$-divergence & $g^*(\zeta)$ & $g(e_\nu)$ & $g^{*'}(\zeta)$ &  $g^{*'-1}(e_\nu)$ \\
         \hline
         \hline
         $\chi^2$ & $\frac{1}{2}(\zeta-1)^2$ & $\frac{1}{2}e_\nu^2+e_\nu$ & $\zeta-1$ & $e_\nu+1$ \\
         KL-divergence & $\zeta\ln{\zeta}-\zeta+1$ & $\exp(e_\nu)-1$ & $\ln{\zeta}$ & $\exp(e_\nu)$ \\
         Reverse-KL & $-\ln{\zeta}+\zeta-1$ & $-\ln(1-e_\nu)$ & $-\frac{1}{\zeta}+1$ & $\frac{1}{1-e_\nu}$ \\
         Hellinger & $\frac{1}{2}(\sqrt{\zeta}-1)^2$ & $\frac{e_\nu}{1-2e_\nu}$ & $\frac{\sqrt{\zeta}-1}{\sqrt{\zeta}}*\frac{1}{2}$ & $\frac{1}{1-2e_\nu}^2$ \\
         Le-Cam & $\frac{1-\zeta}{2(\zeta+1)}+\frac{\zeta-1}{4}$ & $-\sqrt{1-4e_\nu}-e_\nu+1$ & $-\frac{1}{(\zeta+1)^2}+\frac{1}{4}$ & $\sqrt{\frac{4}{1-4e_\nu}}-1$ \\
    \end{tabular}
    \caption{Example divergence functions $g^*$, conjugate $g$, derivative, and inverse of derivative, adjusted to ensure $g^{*'}(1)=0$.}
    \label{tab:f_divergence_table}
\end{table*}

\begin{figure*}[ht]
    \begin{subfigure}[b]{\textwidth}
    \centering
    \includegraphics[width=\linewidth]{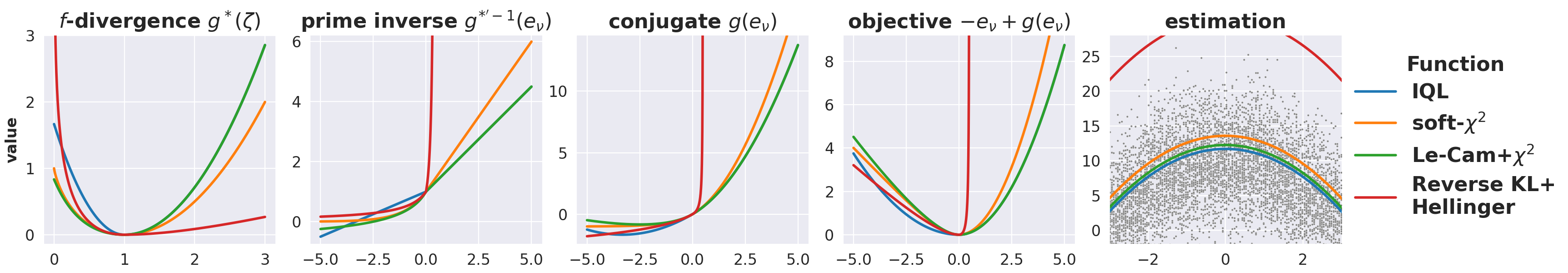}
    \end{subfigure}
    \centering
    \begin{subfigure}[b]{\textwidth}
    \centering
    \includegraphics[width=\linewidth]{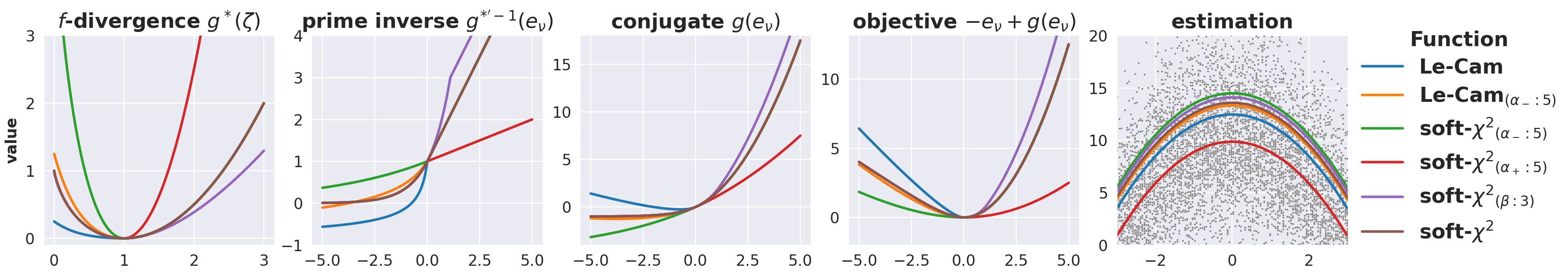}
    \end{subfigure}
    \centering
    \caption{(Top) Example $f$ divergence function. The function for IQL is $\chi^2$ with $\alpha_-=\frac{10}{3}$ and $\alpha_+=\frac{10}{7}$, corresponding to the $70\%$-expectile regression. Le-Cam+$\chi^2$ shares the same coefficient. (Bottom) Illustration of $\alpha_-$, $\alpha_+$, and $\beta$'s effect. Default values are $1.0$. All functions use $\chi^2$ for as $g^*_+(\cdot)$.}
    \label{fig:f_func_diag}
\end{figure*}

Table~\ref{tab:f_divergence_table} provides some examples of existing $f$-divergence approaches, with the adjustment to ensure $g^{*'}(1)=0$, and their corresponding functions. Figure~\ref{fig:f_func_diag} shows graphs of selected flexible $f$-divergence, in terms of the divergence function $g^*(\cdot)$, the prime inverse function $g^{*'-1}(\cdot)$, and conjugate function $g(\cdot)$, as well as the $-e_\nu+g(e_\nu)$ as the loss function in the traditional RL. We also illustrate their estimation effect using a simple two-dimensional dataset with Gaussian noise injected into the y-dimension.
The top row of Figure~\ref{fig:f_func_diag} compares the effect of different base functions. While all of $g^*(\cdot)$ are convex with the minimum value at $g^*(1)=0$ and all of $-e_\nu+g(e_\nu)$ are convex with the minimum value at $-0+g(0)=0$, their derivatives are vastly different and thus have different impact in the optimization effect, as illustrated in the estimation effect. For example, using the Hellinger function as $g^*_+$ will approximate the maximum value. The bottom row of Figure~\ref{fig:f_func_diag} shows the effect of $\alpha_\pm$ and $\beta$. A high-level interpretation of the impact from the perspective of the resulting $-e_\nu+g(e_\nu)$ is that the larger the magnitude of the derivative is at the positive or the negative side, the more penalty the positive or negative Bellman error $e_\nu$ will receive. However, such a difference cannot be captured by the coefficient $\alpha_\pm$ alone since the derivative is often non-linear.

\section{Related works}

\textbf{LP and DRE for offline RL} LP algorithm for RL relies on the concept of density ratio estimation (DRE), which is $\zeta$, to provide estimation or correction for the density ratio between $d$ and $d^D$. DualDICE~\cite{nachum2019dualdice} applied the Fenchel Duality to derive the LP format of RL for the MDP with $0$ reward values for policy evaluation. AlgaeDICE~\cite{nachum2019algaedice} and GenDICE~\cite{zhang2020gendice} derived the LP between $Q$ and $\pi$ for policy improvement. OptiDICE~\cite{lee2021optidice} aimed to resolve the instability from the dual optimization by solving the closed-form solution for $\zeta$. ODICE~\cite{mao2024odice} inspected the gradient of LP. $f$-DVL~\cite{sikchi2023dual} extended the framework of LP and proved that several offline RL algorithms, such as CQL~\cite{kumar2020conservative}, SQL~\cite{xu2023offline}, and XQL~\cite{garg2023extreme} are special cases of LP for RL. Here, we further extend the LP framework by identifying $-e_\nu$ as $L_P$ and relaxing the constraint of $\zeta\geq 0$.

\textbf{$f$-divergence and relaxed regularization} Some functions for the $f$-divergence in prior works are special cases of Ada-$f$. The $f_{\text{soft}-\chi^2}$ function in OptiDICE~\citep{lee2021optidice} uses the $\alpha$-divergence function as $\bar{g}^*_{-}$ and the $\chi^2$-divergence function as $\bar{g}^*_{+}$, with $\alpha_\pm =\beta=1$.
Similarly, RelaxDICE~\citep{yu2023offline} adds a relaxed $f$ divergence upon the base KL-divergence, resulting in the special case with both base functions being $KL$-divergence and $\alpha_{-}=1, \alpha_{+}=2$. 
Indirectly, the expectile regression for updating $V(s)$ in IQL~\citep{kostrikov2021offline} can be considered as optimizing Eq.~\ref{eq:rl_lp_primal_opt_ineq_conjugate_2} directly with $g$ and $g^*$ being $\chi^2$. For a selected expectile value $\tau$, we can use $\alpha_-=\frac{1}{1-\tau}$ and $\alpha_+=\frac{1}{\tau}$ to achieve the same effect. PORelDICE~\cite{kim2024relaxed} relaxes the regularization by setting $\zeta\geq \varepsilon$, which shares a similar effect by using Le-Cam divergence for $\bar{g}^*_{-}$ and set $\alpha_-=\varepsilon$.

\section{Experiments and Analysis}

Our experiments serve two main objectives. The first one is to verify whether the proposed general constrained LP format for RL is a valid RL algorithm. Our second objective is to examine the effect of flexible $f$-divergence in challenging learning settings to see whether it can improve the performance over its base algorithm. Note that the main focus of our experiments is to inspect flexible $f$-divergence's effect by substituting the base algorithm's default $f$-divergence function, either in the explicit form of $g^*(\zeta(s,a))$ or the implicit form of $g(e_\nu(s,a))$. It is not to compare against algorithms with a completely different optimization procedure, though the flexible $f$-divergence could substitute those algorithms' $f$-divergence function as well.

We devised two algorithms following the theoretical framework as \emph{Flex-$f$-Q} and \emph{Flex-$f$-DICE}. \emph{Flex-$f$-Q} approximates $e_\nu$ as $Q_\phi-\nu_\theta$, optimizes Eq.~\ref{eq:rl_lp_primal_opt_ineq_conjugate_2} directly to update $\nu_\theta$ through semi-gradient descent, and use $-e_\nu$ as $L_P$. It is conceptually similar to IQL~\cite{kostrikov2021offline} but with Flex-$f$ as its divergence. \emph{Flex-$f$-DICE} is the OptiDICE algorithm with its divergence substituted with Flex-$f$. Close performance between \emph{Flex-$f$-Q} and IQL can empirically show that the proposed general constrained LP format is a valid RL learning algorithm, with the swapped $L_P$ and removal of $\zeta\geq 0$. Comparison between both pairs of algorithms can show the impact of Flex-$f$.

We focus on environments with continuous action space settings, including continuous control and goal-oriented tasks. For continuous control tasks, we conducted experiments on MuJoCo~\citep{todorov2012mujoco} environments: 
Hopper-v4, Walker2d-v4, Ant-v4, and HalfCheetah-v4. The main objective of these environments is to control various types of simulated robots to move forward as far as possible without falling into an unhealthy state. For goal-oriented tasks, we adopted Fetch environments~\citep{plappert2018multi}, including Push-v2 and PickAndPlace-v2. These environments require the RL agent to control a robot arm and move the target object to the specified goal location. We also included Pen-v1 and Hammer-v1 from ArdoitHand with \emph{cloned} and \emph{human} datasets from D4RL~\citep{fu2020d4rl}.

We collected new datasets for the included MuJoCo and Fetch environments. Using the Soft Actor-Critic~(SAC) algorithm~\citep{haarnoja2018soft}, we trained the behavior policy following D4RL's data collection procedure except for the policy's variance. For goal-oriented environments in Fetch, we used Hindsight Experience Replay (HER)~\citep{andrychowicz2017hindsight} with the future resampling strategy and resampling size $k=4$. 
We limited the explorative behavior of the policy during data collection by fixing the variance of the policy distribution to be $0.0$ as fully deterministic policies. We constructed datasets by mixing data collected from 2, 4, and 10 behavior policies at different expertise levels, referred to as \emph{2-p}, \emph{4-p}, and \emph{10-p}, respectively. Please see the supplementary material for details about the data collection procedure.

\emph{Flex-$f$-DICE} uses $\chi^2$ as $g^*_{+}$ and KL as $g^*_{-}$ in MuJoCo environments. \emph{Flex-$f$-Q} uses KL as $g^*_{+}$ and $\chi^2$ as $g^*_{-}$ in MuJoCo environments. 
$g^*_{+}$ is $\chi^2$ and $g^*_{-}$ is Le-Cam for all other environments.
Other hyperparameters between the two pairs of compared algorithms, IQL vs. \emph{Flex-$f$-Q} and OptiDICE vs. \emph{Flex-$f$-DICE}, are the same, respectively. We trained each algorithm in each setting for five different seeds and report the average normalized returns.

\begin{table}[t!]
\centering
\caption{Min-Max Normalized Returns}
\label{table:exp_adaf_normalized_return}
\begin{tabular}{l|l|ll|ll}
\toprule
 &   & IQL & Flex-$f$- & Opti- & Flex-$f$- \\
env & dataset & & Q & DICE & DICE\\
\midrule
\multirow[t]{3}{*}{hopper} & 4-p & 62.8 & \textbf{76.4} & 77.9 & \textbf{99.0} \\
 & 2-p & 100.7 & \textbf{106.0} & \textbf{45.2} & 40.8 \\
 & 10-p & 63.5 & \textbf{68.5} & 89.9 & \textbf{94.5} \\
\cline{1-6}
\multirow[t]{3}{*}{walker} & 4-p & \underline{50.3} & \underline{50.7} & 75.7 & \textbf{94.2} \\
 & 2-p & \underline{85.6} & \underline{88.3} & \underline{130.6} & \underline{130.3} \\
 & 10-p & \underline{101.2} & \underline{102.0} & 79.8 & \textbf{86.3} \\
\cline{1-6}
\multirow[t]{3}{*}{ant} & 4-p & 113.8 & \textbf{123.1} & 126.0 & \textbf{136.0} \\
 & 2-p & 120.3 & \textbf{146.1} & 126.6 & \textbf{129.3} \\
 & 10-p & \underline{92.7} & \underline{91.1} & 74.9 & \textbf{81.5} \\
\cline{1-6}
\multirow[t]{3}{*}{half-} & 4-p & \underline{52.3} & \underline{54.7} & 15.9 & \textbf{47.0} \\
{cheetah} & 2-p & \underline{44.9} & \underline{42.6} & 13.4 & \textbf{40.1} \\
 & 10-p & \underline{78.2} & \underline{80.3} & 80.5 & 79.4 \\
\cline{1-6}
\multirow[t]{3}{*}{push} & 4-p & \underline{91.6} & \underline{92.4} & 81.9 & \textbf{88.3} \\
 & 2-p & \underline{95.2} & \underline{95.2} & 87.4 & \textbf{93.0} \\
 & 10-p & \underline{84.1} & \underline{81.3} & \underline{69.5} & \underline{68.4} \\
\cline{1-6}
\multirow[t]{3}{*}{pick} & 4-p & \underline{95.7} & \underline{97.0} & 87.2 & \textbf{90.3} \\
{\&place} & 2-p & \underline{92.0} & \underline{92.7} & \underline{91.6} & \textbf{93.2} \\
 & 10-p & \underline{45.2} & \underline{45.3} & 26.6 & \textbf{37.9} \\
\cline{1-6}
\multirow[t]{2}{*}{pen} & cloned & \underline{56.2} & \underline{58.9} & 22.9 & \textbf{40.4} \\
 & human & \underline{52.9} & \underline{53.4} & 10.2 & \textbf{39.6} \\
\cline{1-6}
\multirow[t]{2}{*}{hammer} & cloned & \textbf{2.2} & \underline{0.9} & \underline{0.2} & \textbf{1.3} \\
 & human & \underline{1.0} & \underline{1.9} & \textbf{3.4} & 0.9 \\
\cline{1-6}
\bottomrule
\end{tabular}
\end{table}

Table~\ref{table:exp_adaf_normalized_return} shows each algorithm's average return for each environment and dataset. We highlight the entry if it is higher than the compared counterpart, and underline the entry if the compared algorithms achieved similar performance ($\Delta\leq5$). \emph{Flex-$f$-Q}'s performance is overall close to IQL, empirically showing the correctness of the proposed general LP form as the Bellman minimization algorithm. In some cases, the performance of \emph{Flex-$f$-Q} is higher than IQL. \emph{Flex-$f$-DICE} can almost always achieve improved performance over OptiDICE. This suggests the potential of Flex-$f$ in increasing its base algorithm's performance. \emph{Flex-$f$-DICE}'s higher performance in Mujoco environments, where it shares the same soft-$\mathcal{X}^2$ base function as OptiDICE, verifies the effect of heuristically estimated $\alpha_\pm$ and $\beta$. We also include the result from TD3BC~\citep{fujimoto2021minimalist} and CQL~\citep{kumar2020conservative} in the supplementary material for comparison.

\begin{figure}[th!]
    \centering
    \includegraphics[width=.3\textwidth]{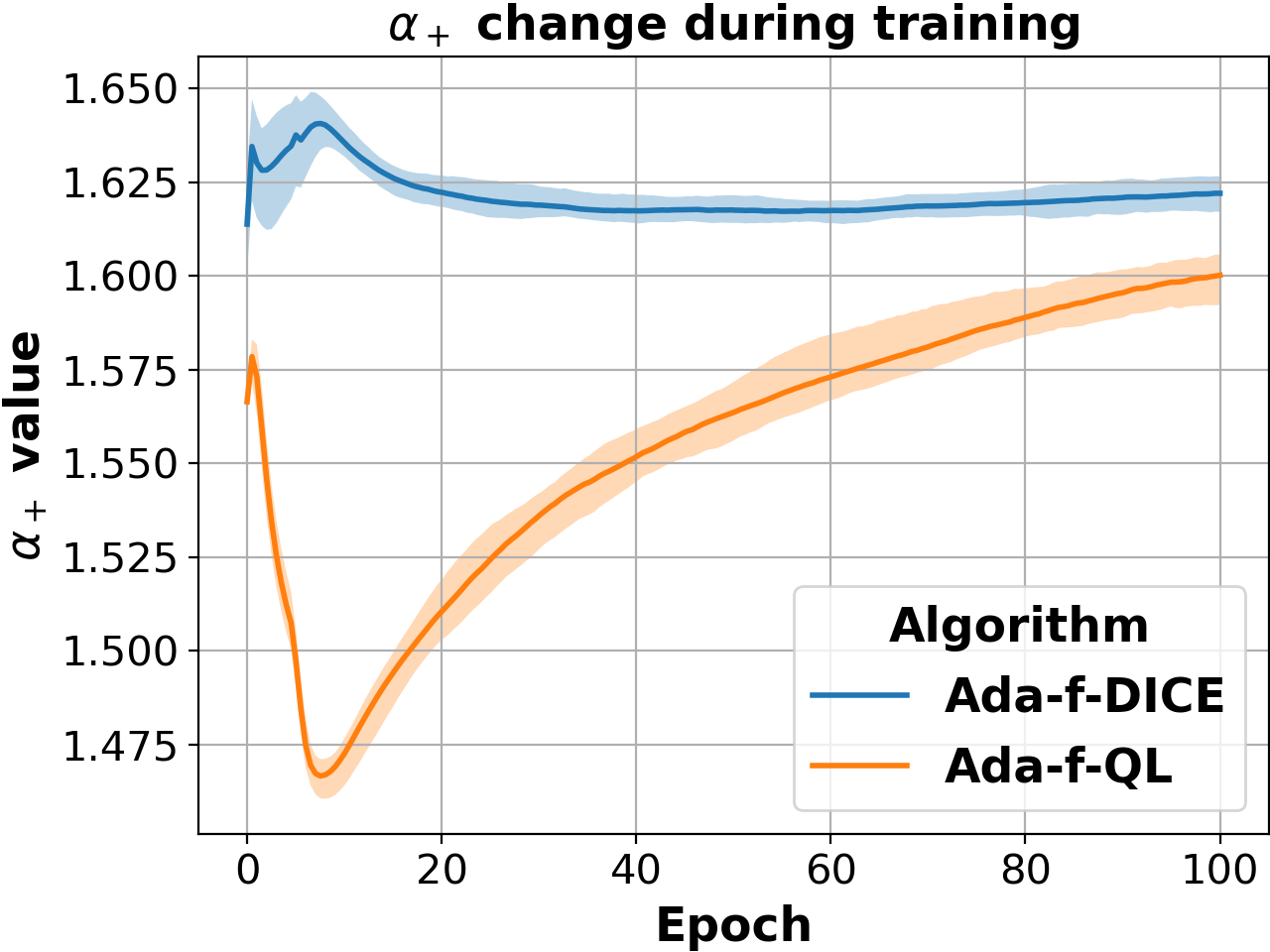}\\
    \includegraphics[width=.3\textwidth]{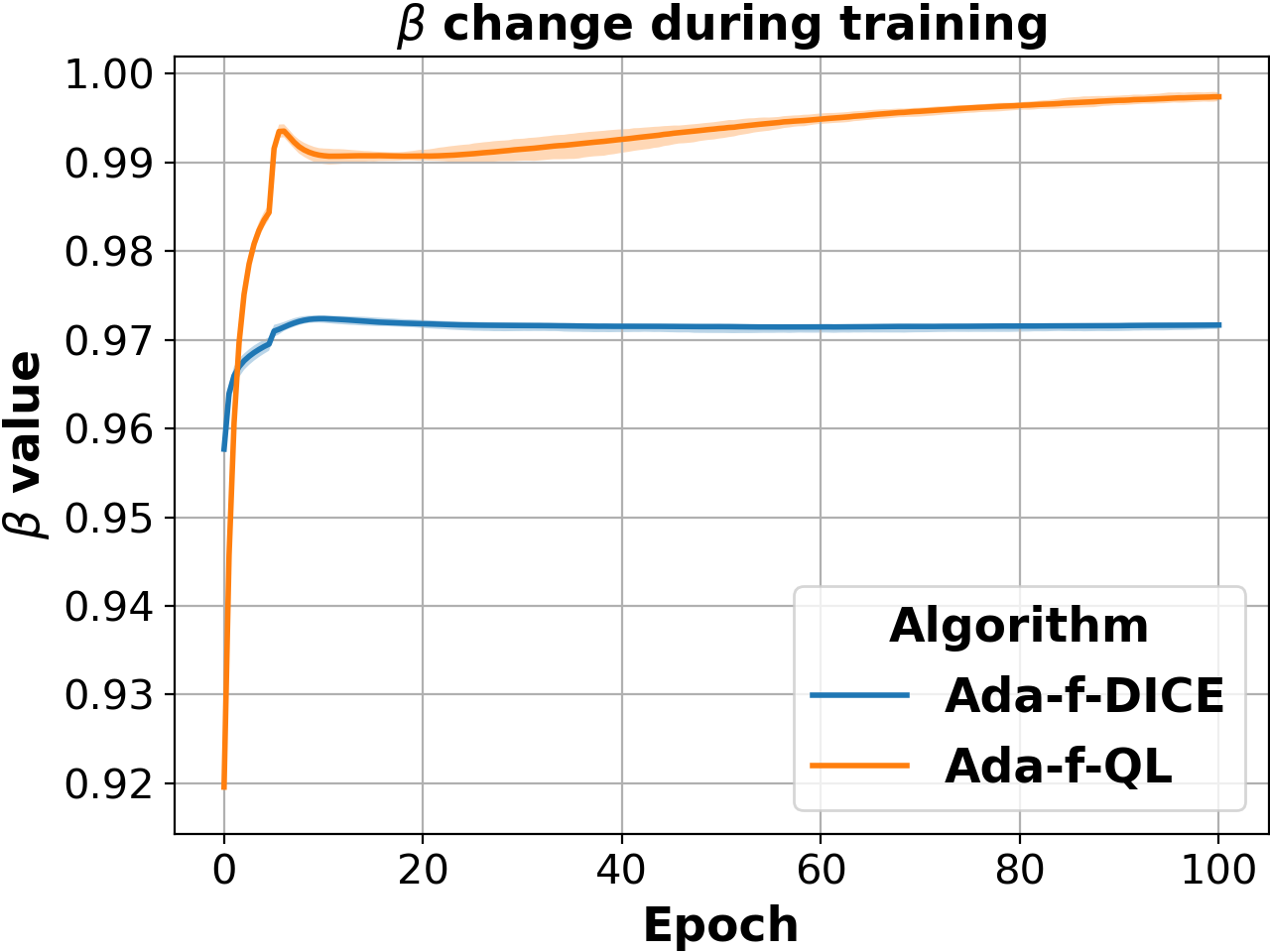}
    \caption{$\alpha_\pm$ and $\beta$ changes throughout training.}
    \label{fig:data_stats_and_algo_perf}
\end{figure}

Figure~\ref{fig:data_stats_and_algo_perf} shows an example of how $\alpha_\pm$ and $\beta$ change throughout the training process. $\alpha_\pm$ can change drastically over the course of training. The decrease of $\alpha_+$ suggests the estimated $e_\nu$ is more aligning with the behavior in the dataset and therefore leading to lower regularization. $\alpha_+$'s increase means the estimated $e_\nu$ diverges away from dataset's behavior, leading to higher regularization. Due to the numeric stability, we limited the $e_\nu$ value for estimation $\beta$ to be in $[-0.2, 0.15]$. But different function approximations and gradient update format still lead to different estimated $\beta$ value.

\begin{table}[t]
\caption{Aplation study of base functions using 4-p}
\label{table:exp_adaf_base_func_ablation}
\begin{tabular}{lllll}
\toprule
 &  & Algorithm & Flex-f-Q & Flex-f-DICE \\
env & $g_+$ & $g_-$ &  &  \\
\midrule
\multirow[t]{8}{*}{hopper} & \multirow[t]{4}{*}{$\mathcal{X}^2$} & $\mathcal{X}^2$ & 54.7 & 75.4 \\
 &  & KL & 71.9 & \textbf{99.0} \\
 &  & Le-Cam & 49.2 & 94.4 \\
 &  & Hellinger & 55.2 & \textbf{100.6} \\
\cline{2-5}
 & \multirow[t]{4}{*}{KL} & $\mathcal{X}^2$ & 76.4 & 81.9 \\
 &  & KL & \textbf{80.8} & 72.3 \\
 &  & Le-Cam & 58.1 & 74.6 \\
 &  & Hellinger & 54.8 & 88.6 \\
\cline{1-5} \cline{2-5}
\multirow[t]{8}{*}{walker} & \multirow[t]{4}{*}{$\mathcal{X}^2$} & $\mathcal{X}^2$ & 45.4 & 97.6 \\
 &  & KL & 55.0 & 94.2 \\
 &  & Le-Cam & 46.0 & 97.8 \\
 &  & Hellinger & 48.7 & \textbf{110.0} \\
\cline{2-5}
 & \multirow[t]{4}{*}{KL} & $\mathcal{X}^2$ & 50.7 & 101.8 \\
 &  & KL & - & 95.1 \\
 &  & Le-Cam & \textbf{60.2} & 96.9 \\
 &  & Hellinger & 45.0 & 91.7 \\
\cline{1-5} \cline{2-5}
\bottomrule
\end{tabular}
\end{table}

We conducted an ablation study to investigate the effect of base functions using Hopper and Walker2D environments with \emph{4-p} datasets. For $g^*_+(\cdot)$, we considered $\mathcal{X}^2$ and KL; for $g^*_-(\cdot)$, we considered $\mathcal{X}^2$, KL, Le-Cam, and Hellinger. Table~\ref{table:exp_adaf_base_func_ablation} shows both \emph{Flex-$f$-Q} and \emph{Flex-$f$-DICE}'s performance. From the high level, Table~\ref{table:exp_adaf_base_func_ablation} shows that there is unlikely a universal function form that works well for every environments because the combination achieving the highest performance is different for each environment and algorithm, though Hellinger-$\mathcal{X}^2$ appears to be relatively consistent for \emph{Flex-$f$-DICE}. This supports our hypothesis that different environments and algorithms call for a different level of constraint and correspondingly different $f$-divergence as optimization objective. 

An interesting observation from Table~\ref{table:exp_adaf_base_func_ablation} is that using KL as $g^*_+(\cdot)$ for \emph{Flex-$f$-Q} only exploded when using KL as $g^*_-(\cdot)$, which is why there is no entry, but was still able to function using other functions as $g^*_-(\cdot)$. This suggests that a different $g^*_-(\cdot)$ can effectively alleviate the risk of Q value overestimation in the semi-gradient style optimization like IQL or \emph{Flex-$f$-Q}.
\section{Conclusion}

In this paper, we identify the challenge of offline RL using datasets with limited stochasticity and a high mixture of expertise levels. 
We propose a general form of LP with expanded options for $L_P$ and the removed constraint of $\zeta\geq 0$. 
We show that such an LP form for RL yields an unconstrained optimization problem that is equivalent to any RL algorithms relying on the Bellman error minimization. We also show a unified form of LP for RL, under which the existing dual optimization problem can be viewed as minimizing the sum of an initial loss term and a penalty function, under the Bellman constraint on either learned value function or density ratio estimation.
We further derived the general function form of the flexible $f$-divergence to provide alternative flexible constraints on the dataset support. Such a flexible $f$-divergence function form opens up future research opportunities of constructing a fully automated adaption algorithm to apply an adaptive constraint based on received data.
Experiment results empirically showed the correctness of the proposed LP formulation and the potential of flexible $f$-divergence in improving algorithm performance with identified challenging datasets. 

We note that using flexible $f$-divergence can introduce new hyperparameters, and deployed heuristic methods for adaptively estimating them during training. Our future work will aim to incorporate the optimization of $\alpha_\pm$ and $\beta$ directly into optimization algorithms, thereby converting the flexible $f$-divergence into a fully optimizable format.

%%%%%%%%%%%%%%%%%%%%%%%%%%%%%%%%%%%%%%%%%%%%%%%%%%%%%%%%%%%%%%%%%%%%%%%%

%%% The next two lines define, first, the bibliography style to be 
%%% applied, and, second, the bibliography file to be used.

\bibliographystyle{ACM-Reference-Format} 
\bibliography{main}

% For arxiv version
\newpage

\appendix

\section{RL Algorithms under General Constrained LP}

Here we describe the adopted RL algorithm for our experiments under the general constrained LP formulation for RL. We omit the input parameter for certain functions when it is duplicating previous statements. We also omit the input and output parameters for the policy since it is always $\pi(a|s)$.
We generalize the notation of the estimation of $e_\nu$ as $\hat{e}$, and use either $\hat{e} = e_\phi$ or $\hat{e} = Q_\phi-\nu_\theta$ to approximate its value. We consider the following optimization problem. 
\begin{align}
    \label{app_eq:gen_rl_apprx_lp_primal}
    &\min_{\theta, \phi} \mathbb{E}_{(s, a)\sim d^{\mathcal{D}}}[\alpha(s)\nu_\theta(s) + g(\hat{e}(s, a))]\\
    \label{app_eq:gen_rl_apprx_lp_primal_constraint}
    &\text{s.t. } \hat{e}(s,a) = e_\theta(s,a), \forall s, a.
\end{align}
The related optimization functions are as follow
\begin{align}
    \label{app_eq:gen_rl_apprx_lag_reord}
    & \max_\zeta[\min_\theta (\alpha\nu_\theta + \zeta\cdot e_\theta) - \max_\phi(\zeta\cdot \hat{e} - g(\hat{e}))] \\
    \label{app_eq:gen_rl_apprx_lag_conj}
    = & \max_\zeta\min_{\theta} \alpha\nu_\theta + \zeta\cdot e_\theta-g^*(\zeta) \\
    \label{app_eq:gen_rl_apprx_nu_final}
    = & \min_\theta \alpha\nu_\theta + g(e_\theta)
\end{align}
We restrict the selection of $LP$ to be either $(1-\gamma)p_0(s)\nu_\theta(s)$, $(1-\gamma)\nu_\theta(s)$, $-e_\theta(s,a)$, or $-\hat{e}(s,a)$. $(1-\gamma)\nu_\theta(s)$ is valid by setting $\alpha = 1-\gamma$, and it can be evaluated to $e_\nu$ under the assumption of $p_0(s)=1$ and the performance difference lemma~\cite{agarwal2019reinforcement}. The last two are valid because of the performance difference lemma and their estimation of true $e_\nu$. To ensure the consistency between compared algorithm, we set $LP$ for \emph{Flex-f-Q} as $\hat{e}(s,a)$, corresponding to IQL, and $LP$ for \emph{Flex-f-DICE} as $(1-\gamma)p_0(s)\nu_\theta(s)$, corresponding to OptiDICE.

For \emph{Flex-f-Q}, its pseudocode for optimization is Algorithm~\ref{algo:ada_f_q}, generalized from semi-gradient descent algorithms such as IQL~\cite{kostrikov2021offline} and $f$-DVL~\cite{sikchi2023dual}. It directly optimizes Eq.~\ref{app_eq:gen_rl_apprx_lp_primal} with $\hat{e}$, and ensure the equality constraint by minimizing the mean squared error between $\hat{e}$ and $e_\theta$. This is a semi-gradient descent optimization for $\nu_\theta$ since its target value is computed as the estimation from another model rather than $\nu_\theta$ itself. We then uses the exponential of estimated advantage for policy extraction in the format of advantage weighted regression~\cite{peng2019advantage}.

\begin{algorithm}[h]
    \caption{Constrained LP Q learning}
    \label{algo:ada_f_q}
    \begin{algorithmic}[1]
        \Require $\mathcal{D}, \nu_\theta, Q_\phi, \pi_\psi, \gamma$
        \For {$i=1, 2, \dots N$}
            \State Sample $\mathcal{B}=\{(s, a, r, s')\sim\mathcal{D}\}$
            \State $\hat{e}(s,a,s') = Q_\phi(s,a)-\nu_\theta(s)$
            \State Train $\nu_\theta$ by minimizing $\mathbb{E}_{\mathcal{B}}[-\hat{e}+g(\hat{e})]$
            \State $e_\theta(s,a,s')=r+\gamma\nu_\theta(s')-\nu_\theta(s)$
            \State Train $Q_\phi$ by minimizing $\mathbb{E}_{\mathcal{B}}[(\hat{e}-e_\theta)^2]$
            \State Train $\pi_\psi$ by minimizing $\mathbb{E}_{\mathcal{B}}[-\exp(\hat{e})\log\pi_\psi]$
        \EndFor
    \end{algorithmic}
\end{algorithm}

For \emph{Flex-f-DICE}, its pseudocode for optimization is Algorithm~\ref{algo:ada_f_dice}, generalized from DICE algorithms such as OptiDICE~\cite{lee2021optidice}. It optimizes $\nu_\theta$ through Eq.~\ref{app_eq:gen_rl_apprx_nu_final}. The Lagrange version of optimizing $e_\phi$ in OptiDICE would correspond to Eq.~\ref{app_eq:gen_rl_apprx_lag_reord}. We adopted the MSE version of optimizing $e_\phi$ as we empirically found it more stable. The policy extraction follows OptiDICE, which uses the ReLU wrapped predicted $\zeta^*(s,a)=g^{*'-1}(s, a)$ as weights for the AWR style policy extraction. This is a crucial difference as we only enforce the constraint of $\zeta\geq0$ during the policy extraction, but not during the training of either $\nu_\theta$ or $e_\phi$. $\mathcal{D}_0=\{s^i_0\}_{i=0}^N$ is the set of initial states for estimating the initial state distribution $p_0$.

\begin{algorithm}[h]
    \caption{Constrained LP DICE}
    \label{algo:ada_f_dice}
    \begin{algorithmic}[1]
        \Require $\mathcal{D}, \mathcal{D}_0=\{s^i_0\}_{i=0}^N, \nu_\theta, e_\phi, \pi_\psi, \gamma$
        \For {$i=1, 2, \dots N$}
            \State Sample $\mathcal{B}=\{(s, a, r, s')\sim\mathcal{D}\}$
            \State Sample $\mathcal{B}_0 = \{s_0 \sim \mathcal{D}_0\}$
            \State $e_\theta(s,a,s')=r+\gamma\nu_\theta(s')-\nu_\theta(s)$
            \State Train $\nu_\theta$ by minimizing $\mathbb{E}_{\mathcal{B}_0}[(1-\gamma)\nu_\theta]+\mathbb{E}_{\mathcal{B}}[g(e_\theta)]$
            \State Train $e_\phi$ by minimizing $\mathbb{E}_{\mathcal{B}}[(e_\phi-e_\theta)^2]$
            \State $w_\phi(s,a) = \max(g^{*'-1}(e_\phi(s,a)), 0)$
            \State Train $\pi_\psi$ by minimizing $\mathbb{E}_{\mathcal{B}}[-w_\phi\log\pi_\psi]$
        \EndFor
    \end{algorithmic}
\end{algorithm}

Algorithm~\ref{algo:ada_f_param_est} shows the pseudocode for estimating parameters, $\alpha_\pm$ and $\beta$, of Ada-$f$. We perform the parameter estimation using the predicted $\hat{e}$ from the optimization algorithm to avoid duplicated model forward calls. In practice, we would want to avoid the $\alpha_+$ to be too high. Therefore, we smooth the computation of cosine similarity $\delta$ by setting a lower bound $\iota_b$. We also clip the predicted estimation $\hat{e}$ to avoid it being out of the domain of $g(\cdot)$, like the conjugate function for the Le-Cam divergence. We also performed the exponential moving average to smooth the estimation. 

\begin{algorithm}[h]
    \caption{Ada-f parameter estimation}
    \label{algo:ada_f_param_est}
    \begin{algorithmic}[1]
        \Require $\mathcal{D}, \pi_b, \iota_b$
        \For {$i=1, 2, \dots N$}
            \State Sample $\mathcal{B}=\{(s, a, r, s')\sim\mathcal{D}\}$
            \State Train $\pi_b$ by minimizing $\mathbb{E}_{\mathcal{B}}[-\log\pi_b]$
            \State Compute $\hat{e}$ from training with Algo.~\ref{algo:ada_f_q} or~\ref{algo:ada_f_dice} and $\mathcal{B}$
            \State $\delta = \cos(\overrightarrow{\pi_b(a|s)}, \overrightarrow{\exp(\hat{e}(s, a)})\cdot(1-\iota_b) + \iota_b$
            \State $\alpha_+ = \frac{1}{\delta}$
            \State $\alpha_- = \frac{1}{1-\delta}$
            \State $\beta = \mathbb{E}_{\mathcal{B}}[\hat{e}]$
        \EndFor
    \end{algorithmic}
\end{algorithm}

\section{Equivalence of Existing RL algorithm in the General constrained LP formulation}

Here we show some examples of existing RL algorithm in the equivalent general constrained LP formulation, with the incorporation of Ada-$f$. We start from the unconstrained minimization problem Eq.~\ref{app_eq:gen_rl_apprx_nu_final}, which we have shown to share the same solution to the constrained format from Proposition 1.

\subsection{XQL~\cite{garg2023extreme}}

The objective function for updating $V$ in XQL is

\begin{align}
    J_{\text{XQL}}(V) = \mathbb{E}_{(s, a)\sim\mathcal{D}}[e^{\hat{Q}(s,a)-V(s)} - (\hat{Q}(s,a)-V(s))-1]
\end{align}

Substituting $\hat{Q}(s,a)-V(s)$ with $e_\nu$ and selecting $g(\cdot)=\exp(\cdot)$, i.e. the conjugate function for KL-divergence function, we have

\begin{align}
    J_{\text{XQL}}(V) = \mathbb{E}_{(s, a)\sim\mathcal{D}}[g(e_\nu)-e_\nu] - 1
\end{align}
This is the same objective function as Eq.~\ref{app_eq:gen_rl_apprx_nu_final} by evaluating $\alpha\nu$ for $L_P$ as $-e_\nu$ and subtracting $1$. 

\subsection{IQL~\cite{kostrikov2021offline}}

The objective function for updating $V$ in IQL is

\begin{align}
    J_{\text{IQL}}(V) = \begin{cases}
        \tau(\hat{Q}(s, a)-V(s))^2 & \hat{Q}(s, a)-V(s) \geq 0 \\
        (1-\tau)(\hat{Q}(s, a)-V(s))^2 & \hat{Q}(s, a)-V(s) < 0 \\
    \end{cases}
\end{align}

We can select $L_P = -e_\nu$, picking both $\bar{g}^*_{\pm}$ as $\chi^2$-divergence function, $\beta=1$, $\alpha_+=\frac{1}{\tau}$ and $\alpha_-=\frac{1}{1-\tau}$. We solve for the $\alpha_+$ side first, following the conjugate function, which yields

\begin{align}
    J_{\text{IQL}}(V; e_\nu\geq0) & = -e_\nu + e_\nu\cdot g^{*'-1}(\frac{e_\nu}{\alpha_+}) \notag\\
    & - g^*(g^{*'-1}(\frac{e_\nu}{\alpha_+})) \\
    & = -e_\nu + e_\nu \cdot (\frac{e_\nu}{\alpha_+}+1) \notag \\
    & -\alpha_+\frac{1}{2}(\frac{e_\nu}{\alpha_+}+1 - 1)^2 \\
    & = \frac{1}{\alpha_+}\frac{1}{2}e_\nu^2
\end{align}

We can derive the same equation for $\alpha_-$. Substituting 

we can rewrite Eq~\ref{app_eq:gen_rl_apprx_nu_final} with $\alpha_+=\frac{1}{\tau}$ and $\alpha_-=\frac{1}{1-\tau}$, 

\begin{align}
    J_{\text{IQL}}(V) = \begin{cases}
        \tau\frac{1}{2}e_\nu^2 & e_\nu \geq 0 \\
        (1-\tau)\frac{1}{2}e_\nu^2 & e_\nu < 0 \\
    \end{cases}
\end{align}

This is equivalent to IQL's objective with the coefficient of $\frac{1}{2}$.

\section{Implementation details}

Here we describe the implementation details of compared algorithms, including model architectures and other hyperparameters.

\begin{table}[h]
\caption{IQL hyperparameters}
\label{table:iql_hyper}
\centering
\begin{tabular}{llllr}
\toprule
 & $\tau$ & $\beta$ & $c_r$ & $B$ \\
Environment &  &  &  &  \\
\midrule
Ant-v4 & 0.7 & 5.0 & 0.257 & 256 \\
HalfCheetah-v4 & 0.9 & 3.0 & 0.080 & 256 \\
Hopper-v4 & 0.7 & 3.0 & 0.300 & 256 \\
Walker2d-v4 & 0.7 & 5.0 & 0.217 & 256 \\
FetchPickAndPlace-v2 & 0.7 & 5.0 & 1.000 & 512 \\
FetchPush-v2 & 0.7 & 5.0 & 1.000 & 256 \\
AdroitHandPen-v1 & 0.7 & 1.0 & 1.000 & 256 \\
AdroitHandHammer-v1 & 0.7 & 1.0 & 1.000 & 256 \\
\bottomrule
\end{tabular}
\end{table}

\textbf{IQL} in our experiments uses $2$ dense layers with sizes of $256$ and ReLU activation function for critic, value function, and actor model architectures. We trained the critic and the value function using the Adam optimizer with a learning rate of $3\times 10^{-4}$. We use the same optimizer except the $\alpha=0.001$ for the actor optimizer's cosine decay learning rate scheduler since we trained the algorithm for $1\times 10^{6}$ gradient steps. We use a larger batch size for \emph{FetchPickAndPlace-v2}. Table~\ref{table:iql_hyper} lists the corresponding values of expectile $\tau$, temperature $\beta$, reward scale $c_r$, and batch size $B$ for each environment. We keep other hyperparameters and training settings the same as the original paper~\citep{kostrikov2021offline}.

\textbf{Flex-f-Q} uses a reward scale of $0.1$ for Mujoco and AdroitHand environments, and $1.0$ for Fetch environments. All $\bar{g}^*_{+}$ is $\chi^2$ and $\bar{g}^*_{-}$ is Le-Cam. The lowest similarity bound $\iota_b$ for computing $\alpha_\pm$ is $0.3$ and the clipping of $\hat{e}$ for computing $\beta$ is $[-0.2, 0.15]$. The batch size is 512 for all environments. All other parameters are the same as IQL.

\textbf{OptiDICE} in our experiments uses $2$ dense layers with sizes of $256$ and ReLU activation function for $\nu_\theta$, $e_\phi$, and actor model architectures. We trained all models using the Adam optimizer with a learning rate of $3\times 10^{-4}$, with separated optimizers, one for each model. We picked $\alpha = 0.1$ and soft-$\chi^2$ divergence for all environments. We adopted the MSE optimization for $e_\phi$ and the AWR style policy extraction. The batch size for all environments is 512. The reward scale is $1.0$. We performed standardization on the observation values as we find it improve OptiDICE's performance.

\textbf{Flex-f-DICE} uses a reward scale $1.0$ for all environments, but with the global weighting of $g(\cdot)$ as $\alpha=0.1$ to be consistent with OptidICE. All $\bar{g}^*_{+}$ is $\chi^2$ and $\bar{g}^*_{-}$ is Le-Cam, except for Mujoco environments where $\bar{g}^*_{-}$ is KL. The lowest similarity bound $\iota_b$ for computing $\alpha_\pm$ is $0.3$ and the clipping of $\hat{e}$ for computing $\beta$ is $[-0.2, 0.15]$. The batch size is 512 for all environments. All other parameters are the same as OptiDICE.

\textbf{TD3BC} in our experiments uses $2$ dense layers with sizes of $256$ and ReLU activation function for both critic and actor models. We trained all models using the Adam optimizer with a learning rate of $3\times 10^{-4}$, with separated optimizers, one for each model. The temperature for updating the policy is $2.5$. The action noise sampling for $Q$-value update has a $0.2$ standard deviation and is clipped between $\pm0.5$. The batch size follows the same setting as for IQL. The reward scale is $1.0$. 

\textbf{CQL} in our experiments uses $2$ dense layers with sizes of $256$ and ReLU activation function for both critic and actor models. We trained all models using the Adam optimizer with a learning rate of $3\times 10^{-4}$, with separated optimizers, one for each model. We deployed CQL with the maximum $Q$ backup and the importance sampling for updating $Q$ values. The batch size follows the same setting as for IQL. The reward scale is $1.0$. For the AdroitHand environments, we deployed the lagrange version of CQL to prevent exploding estimation of $Q$ values.

\begin{table*}[t!]
\caption{Full normalized Returns with standard deviations}
\label{table:exp_adaf_normalized_return_full_std}
\begin{tabular}{llllllll}
\toprule
 & & CQL & TD3BC & IQL & Flex-$f$- & Opti- & Flex-$f$- \\
env & dataset &  &  &  & Q & DICE & DICE \\
\midrule
\multirow[t]{3}{*}{hopper} & 4-p & 37.6 $\pm$ 9.34 & 32.8 $\pm$ 20.68 & 62.8 $\pm$ 9.83 & 76.4 $\pm$ 5.94 & 77.9 $\pm$ 10.94 & 99.0 $\pm$ 8.52 \\
 & 2-p & 34.8 $\pm$ 2.09 & 23.5 $\pm$ 12.68 & 100.7 $\pm$ 2.73 & 106.0 $\pm$ 2.58 & 45.2 $\pm$ 19.60 & 40.8 $\pm$ 18.31 \\
 & 10-p & 20.0 $\pm$ 14.04 & 36.6 $\pm$ 8.72 & 63.5 $\pm$ 6.67 & 68.5 $\pm$ 5.82 & 89.9 $\pm$ 8.81 & 94.5 $\pm$ 3.90 \\
\cline{1-8}
\multirow[t]{3}{*}{walker} & 4-p & 9.9 $\pm$ 1.69 & 9.3 $\pm$ 3.87 & 50.3 $\pm$ 6.05 & 50.7 $\pm$ 3.76 & 75.7 $\pm$ 13.49 & 94.2 $\pm$ 5.48 \\
 & 2-p & 4.3 $\pm$ 0.30 & 17.2 $\pm$ 14.32 & 85.6 $\pm$ 13.38 & 88.3 $\pm$ 10.03 & 130.6 $\pm$ 2.49 & 130.3 $\pm$ 3.36 \\
 & 10-p & 83.4 $\pm$ 12.57 & 89.5 $\pm$ 7.59 & 101.2 $\pm$ 6.39 & 102.0 $\pm$ 3.45 & 79.8 $\pm$ 10.74 & 86.3 $\pm$ 5.93 \\
\cline{1-8}
\multirow[t]{3}{*}{ant} & 4-p & -13.8 $\pm$ 27.97 & 53.7 $\pm$ 14.54 & 113.8 $\pm$ 5.88 & 123.1 $\pm$ 4.43 & 126.0 $\pm$ 6.53 & 136.0 $\pm$ 2.36 \\
 & 2-p & 67.0 $\pm$ 4.38 & 17.3 $\pm$ 37.05 & 120.3 $\pm$ 7.63 & 146.1 $\pm$ 3.43 & 126.6 $\pm$ 4.12 & 129.3 $\pm$ 10.06 \\
 & 10-p & 25.8 $\pm$ 11.08 & 32.3 $\pm$ 25.87 & 92.7 $\pm$ 0.73 & 91.1 $\pm$ 1.57 & 74.9 $\pm$ 6.00 & 81.5 $\pm$ 1.48 \\
\cline{1-8}
\multirow[t]{3}{*}{halfcheetah} & 4-p & 5.1 $\pm$ 2.69 & 20.9 $\pm$ 2.59 & 52.3 $\pm$ 1.42 & 54.7 $\pm$ 3.08 & 15.9 $\pm$ 3.22 & 47.0 $\pm$ 2.70 \\
 & 2-p & 3.5 $\pm$ 1.39 & 19.2 $\pm$ 2.11 & 44.9 $\pm$ 1.34 & 42.6 $\pm$ 4.60 & 13.4 $\pm$ 2.22 & 40.1 $\pm$ 13.96 \\
 & 10-p & 16.3 $\pm$ 24.04 & 72.9 $\pm$ 2.06 & 78.2 $\pm$ 0.15 & 80.3 $\pm$ 0.35 & 80.5 $\pm$ 0.57 & 79.4 $\pm$ 0.71 \\
\cline{1-8}
\multirow[t]{3}{*}{push} & 4-p & 74.7 $\pm$ 1.87 & 98.9 $\pm$ 1.22 & 91.6 $\pm$ 2.45 & 92.4 $\pm$ 1.59 & 81.9 $\pm$ 2.79 & 88.3 $\pm$ 4.42 \\
 & 2-p & 81.5 $\pm$ 3.37 & 64.9 $\pm$ 5.46 & 95.2 $\pm$ 1.70 & 95.2 $\pm$ 2.17 & 87.4 $\pm$ 3.60 & 93.0 $\pm$ 1.87 \\
 & 10-p & 76.1 $\pm$ 3.77 & 88.1 $\pm$ 10.59 & 84.1 $\pm$ 2.30 & 81.3 $\pm$ 3.60 & 69.5 $\pm$ 2.14 & 68.4 $\pm$ 3.78 \\
\cline{1-8}
\multirow[t]{3}{*}{pickandplace} & 4-p & 65.1 $\pm$ 3.77 & 70.7 $\pm$ 5.35 & 95.7 $\pm$ 1.33 & 97.0 $\pm$ 1.25 & 87.2 $\pm$ 1.36 & 90.3 $\pm$ 2.50 \\
 & 2-p & 49.0 $\pm$ 3.39 & 48.9 $\pm$ 4.73 & 92.0 $\pm$ 3.75 & 92.7 $\pm$ 2.78 & 91.6 $\pm$ 1.31 & 93.2 $\pm$ 2.14 \\
 & 10-p & 9.2 $\pm$ 2.90 & 43.3 $\pm$ 3.37 & 45.2 $\pm$ 3.76 & 45.3 $\pm$ 4.03 & 26.6 $\pm$ 4.06 & 37.9 $\pm$ 3.12 \\
\cline{1-8}
\multirow[t]{2}{*}{pen} & cloned & -1.1 $\pm$ 1.02 & - & 56.2 $\pm$ 5.77 & 58.9 $\pm$ 1.53 & 22.9 $\pm$ 4.07 & 40.4 $\pm$ 4.36 \\
 & human & 6.7 $\pm$ 3.42 & - & 52.9 $\pm$ 3.30 & 53.4 $\pm$ 2.99 & 10.2 $\pm$ 10.46 & 39.6 $\pm$ 4.96 \\
\cline{1-8}
\multirow[t]{2}{*}{hammer} & cloned & 0.3 $\pm$ 0.09 & - & 2.2 $\pm$ 0.74 & 0.9 $\pm$ 0.17 & 0.2 $\pm$ 0.09 & 1.3 $\pm$ 0.75 \\
 & human & 1.1 $\pm$ 0.35 & - & 1.0 $\pm$ 0.56 & 1.9 $\pm$ 0.60 & 3.4 $\pm$ 0.80 & 0.9 $\pm$ 0.29 \\
\cline{1-8}
\bottomrule
\end{tabular}
\end{table*}

\section{Environments and Datasets}

\subsection{Experiment environments}

We provide additional description and corresponding experiment settings for each family of environments in our experiments. 

\paragraph{Gym-MuJoCo}\citep{todorov2012mujoco} is a set of widely included testbed environments for both online and offline reinforcement learning. We choose four continuous control tasks, \emph{Hopper-v4}, \emph{Walker2d-v4}, \emph{Ant-v4}, and \emph{HalfCheetah-v4}. They consist of manipulating different simulated 3D robots to move forward in the environment while maintaining a healthy state, except for \emph{HalfCheetah-v4} which is always healthy. The trajectory will terminate when the robot is in an unhealthy state so there is a healthy reward at each time step as long as the trajectory goes on. This creates a weak correspondence between the length of the trajectory and its noiseless rating for \emph{Hopper-v4}, \emph{Walker2d-v4} and \emph{Ant-v4} environments. We collected trajectories truncated at $1000$ steps except those terminated early.

\paragraph{Fetch}\citep{plappert2018multi} is a set of goal-oriented robotic environments with a sparse reward when the robot satisfies the goal condition. We include two Fetch environments in our study, i.e. \emph{FetchPush-v2} and \emph{FetchPickAndPlace-v2}. Their high level objectives are to control a simulated Fetch Mobile Manipulator robot arm with a two-fingered gripper and move an object to the designated goal location under certain restrictions. In our study, the trajectory in these environments never terminates and will only be truncated when reaching the maximum time steps (50). So the robot arm needs to maintain its endpoint or the object at the goal location to receive maximum return. Following \citet{plappert2018multi}'s setting, we assign $0$ as the reward when the robot arm agent meets the goal condition and $-1$ as the reward otherwise. Comparing to the traditional $1$-$0$ reward scheme for the goal-oriented environment, it allows the trajectory return and consequently the corresponding rating to reflect the velocity of the agent satisfying the goal condition. 

\paragraph{AdroidHand}\citep{rajeswaran2017learning} environments simulate a Shadow Dexterous Hand robot performing various tasks. The task is considered successful when certain goal conditions are met for that task. For \emph{AdroitHandPen-v1}, the task is to reposition a pen object to the target orientation. For \emph{AdroitHandHammer-v1}, the task is to control the simulated robot hand to grab a hammer and use that hammer to drive a nail to a board. In our study, the trajectory in these environments never terminates and will only be truncated when reaching the maximum time steps (200). While these environments can be categorized as goal-oriented, we adopt the dense reward version of included environments due to their high difficulty. 

\subsection{Dataset collection}

We collected new datasets for included Mujoco and Fetch environments. We trained the behavior policy using the Soft Actor Critic (SAC) algorithm~\citep{haarnoja2018soft} following D4RL's data collection procedure except for the policy's variance. For goal-oriented environments in Fetch, we used Hindsight Experience Replay (HER)~\citep{andrychowicz2017hindsight} with the future resampling strategy and resampling size $k=4$. 
We limited the explorative behavior of the policy during data collection by fixing the variance of the policy distribution to be $0.0$ as fully deterministic policies. We constructed datasets by mixing data collected from 2, 4, and 10 behavior policies at different expertise levels, referred to as \emph{2-p}, \emph{4-p}, and \emph{10-p} respectively.  
\begin{enumerate}
    \item \emph{2-p} follows the same composition as in D4RL that contains 50\% of the data from a policy at 30\% to 50\% of expert performance level and 50\% of the data from the expert policy as the converged optimal policy during training.%, referred to as \emph{2-p}.
    \item \emph{4-p} contains data from the expert policy, \emph{sub-optimal} at 60\%, \emph{medium} at 30\%, and \emph{low} at 10\% of the expert performance.%, referred to as \emph{4-p}.
    \item \emph{10-p} contains data from 10 equally distributed checkpoints of a policy trained to near-expert level (at 90\% of expert performance).%, referred to as \emph{10-p}.
\end{enumerate}

\section{Additional results}

Table~\ref{table:exp_adaf_normalized_return_full_std} shows full experiment results with standard deviation. It also includes additional experiment results from CQL~\cite{kumar2020conservative} and TD3BC~\cite{fujimoto2021minimalist}. We excluded TD3BC's results in AdroidHand environments because we could not complete its training due to the exploding $Q$-value estimation.

%\bibliography{main}

%%%%%%%%%%%%%%%%%%%%%%%%%%%%%%%%%%%%%%%%%%%%%%%%%%%%%%%%%%%%%%%%%%%%%%%%

\end{document}